%% file: main.tex
\documentclass{article}

\usepackage{mathrsfs,amsmath,amsfonts,amssymb,amstext,amscd, dsfont,pifont,
            amsthm,euscript,color,xcolor,accents,xr} %dsfont,pifont,stmaryrd
\usepackage{algorithmic}
\usepackage{algorithm}
\usepackage{url}
\usepackage{mathtools}
\usepackage{epsfig}
\usepackage{float}
\usepackage{appendix}
\usepackage{floatflt}
\usepackage{nicefrac}
\usepackage{anysize}
\usepackage{enumerate}
\usepackage{epsfig}
\usepackage{graphicx}
\usepackage{mathtools}
\usepackage{upgreek}
\usepackage{pdfpages}
\usepackage{subcaption}  
\usepackage{bbold}
\usepackage{bbm,bm}

\usepackage[linecolor=magenta!60!, backgroundcolor=magenta!10!,textwidth=1.6cm, textcolor=magenta]{todonotes}

\numberwithin{equation}{section} 

\newcommand{\heavy}{\theta} %{\hat \theta^{HB}}

\newcommand*{\Scale}[2][4]{\scalebox{#1}{$#2$}}%
\def\nU{n_{\Scale[0.4]{\mathcal{U}}}}
\def\nX{n_{\Scale[0.4]{\mathcal{X}}}}

\input{macros_nicole}

\usepackage[onehalfspacing]{setspace} %singlespacing, doublespacing

\definecolor{citecolor}{rgb}{0, 0.3, 0.6}
\usepackage[colorlinks=true,linkcolor=citecolor, citecolor=citecolor, backref]{hyperref}
\usepackage[round]{natbib}
\newcommand{\mycite}[1]{\cite{#1}} %{\textbf{\mycite{#1}}}
\newcommand{\mycitep}[1]{\citep{#1}}

\title{Random feature approximation for general spectral methods}

\author{Mike Nguyen\footnote{Corresponding Author} \\
Technical University  of Braunschweig \\
\texttt{mike.nguyen@tu-braunschweig.de} 
\and
Nicole M\"ucke\\
Technical University  of Braunschweig  \\ 
\texttt{nicole.muecke@tu-braunschweig.de}
}
\date{\today}

\begin{document}

\maketitle

\begin{abstract}

Random feature approximation is arguably one of the most widely used techniques for kernel methods in large-scale learning algorithms. In this work, we analyze the generalization properties of random feature methods, extending previous results for Tikhonov regularization to a broad class of spectral regularization techniques. This includes not only explicit methods but also implicit schemes such as gradient descent and accelerated algorithms like the Heavy-Ball and Nesterov method. Through this framework, we enable a theoretical analysis of neural networks and neural operators through the lens of the Neural Tangent Kernel (NTK) approach trained via gradient descent.  For our estimators we obtain optimal learning rates over regularity classes (even for classes that are not included in the reproducing kernel Hilbert space), which are defined through appropriate source conditions. This improves or completes previous results obtained in related settings for specific kernel algorithms.

\end{abstract}


%%%%%%%%%%%%%%%%%%%%%%%%%%%%%%%%%%%%%%%%%%%%%%%%%%%%%%%%%%%%%%%%%%%%%%%%%%%%%%%%%%%%%%%
%%%%%%%% SECTION 1
%%%%%%%%%%%%%%%%%%%%%%%%%%%%%%%%%%%%%%%%%%%%%%%%%%%%%%%%%%%%%%%%%%%%%%%%%%%%%%%%%%%%%%%

\section{Introduction}

The rapid technological progress has led to accumulation of vast amounts of high-dimensional data in recent years. Consequently, to analyse such amounts of data it is no longer sufficient to create algorithms that solely aim for the best possible predictive accuracy. Instead, there is a pressing need to design algorithms that can efficiently process large datasets while minimizing computational overhead. In light of these challenges, two fundamental algorithmic tools, fast gradient methods, and sketching techniques, have emerged. Iterative gradient methods such as acceleration methods \mycitep{pagliana2019implicit, 10.1214/18-EJS1395, 7330562, JMLR:v18:16-410} or stochastic gradient methods \mycitep{SGDfeatures, NIPS2012_905056c1, doi:10.1137/140961791} leading to favorable convergence rates while reducing computational complexity during learning. On the other hand sketching techniques enable the reduction of data dimension, thereby decreasing memory requirements through random projections. The allure of combining both methodologies has garnered significant attention from researchers and practitioners alike. Especially for Kernel based algorithms various sketching tools have gained a lot of attention in recent years. For non-parametric statistical approaches kernel methods are in many applications still state of the art and provide an elegant and effective framework to develop theoretical optimal learning bounds \mycitep{Muecke2017op.rates, Lin_2020, spectral.rates, zhang2024optimalitymisspecifiedspectralalgorithms}. However those benefits come with a computational cost making these methods unfeasible when dealing with large datasets. In fact, traditional kernelized learning algorithms require storing the kernel Gram matrix $\mathbf{K}\in\mathbb{R}^{n\times n}$, where $\mathbf{K}_{i,j}=K(x_i,x_j)$, $\,K(.,.)$ denotes the kernel function and $x_i,x_j$ the data points. This results in a memory cost of at least $O(n^2)$ and a time cost of up to $O(n^3)$, where $n$ denotes the data set size \mycitep{kernellearning}. Most popular sketching tools to overcome these issues are Nyström approximations and random feature approximation (RFA). 
The Nyström approximation can allow a significant speed-up of the computations. This speed-up is achieved by using, instead of the kernel matrix $\mathbf{K}$, an approximation matrix $\hat{\mathbf{K}}$ of lower rank $q$. An advantage of the method is that it is not necessary to compute or store the whole kernel matrix, but only a sub-matrix of size $q\times n$. It reduces the storage and complexity requirements to $O(nq)$ and $O(nq^2)$ respectively \mycitep{rudi2016more, NIPS2000_19de10ad, pmlr-v9-cortes10a}
Numerous methods have been proposed for approximating the kernel matrix, including techniques such as PCA-based approximation \mycitep{NIPS2000_19de10ad, JMLR:v23:21-0766,pmlr-v108-sterge20a} and stochastic approximation methods \mycitep{6918503, drineas2005nystrom, pmlr-v108-sterge20a}.
Depending on the eigenvalue decay of the kernel matrix, $q$ may need to be quite large to achieve a good approximation. Computing a PCA in such cases can itself be computationally expensive. Therefore, if certain structural assumptions about the kernel can be made, it is often practical to consider RFA as an alternative to the Nyström method. In this paper, we investigate algorithms, using the interplay of fast learning methods and RFA and analyze generalization performance of such algorithms. 
The concept of random feature approximation relies on the assumption that the kernel admits a specific integral representation. The integral is then approximated using a summation of so-called random features. This approach reduces memory and computational costs to $O(nM)$ and $O(nM^2)$, respectively, where $M$ denotes the number of random features. The first question that naturally arises is how many random features are required to achieve optimal convergence rates when using the approximated kernel with traditional kernel estimation methods. For the kernel ridge regression (KRR) method, this question has already been extensively studied \mycitep{NIPS2007_013a006f, li2021unified, pmlrv119zhen20a, features, lanthaler2023error}. 
In 2007, \cite{NIPS2007_013a006f} established optimal convergence rates for $M=O(n)$ random features in the case of real-valued kernels (rvk). Later, \mycite{features} improved this result to $M=O(\sqrt{n}\,\log(n))$. This result was further extended to stochastic kernel ridge regression in \mycite{SGDfeatures}. More recently, \mycite{lanthaler2023error} succeeded in eliminating the logarithmic factor, demonstrating the rate $M=O(\sqrt{n})$ for more general vector-valued kernels (vvk). We refer to Table \ref{table3} for a comparison. \vspace{0.2cm}

\textbf{Contribution}
To the best of our knowledge, RFA has so far been studied primarily for KRR. Using a general spectral filtering framework \cite{Caponetto}, we establish fast convergence rates for a wide range of learning methods with implicit or explicit regularization. This includes methods such as gradient descent, acceleration techniques, and also covers prior results of \cite{features} for KRR and vector-valued kernels. Furthermore, we addressed the saturation effect inherent to the kernel ridge regression estimator, which limits the convergence rates in all previously mentioned work. By overcoming this limitation, we achieved fast convergence rates for objectives with any degree of smoothness. Additionally, our framework allows for kernels represented as sums of integral kernels, which include, among others, operator-valued neural tangent kernels. As a result, this paper lays the foundation for deriving convergence rates for Neural Operators \cite{nguyen2024optimalconvergenceratesneural}.\\ For more details we refer to Section \ref{sec:main-results}.

\begin{table}[H]
\caption{The first column compares the number of random features, needed to provide a generalization 
bound of order $O(n^{-\frac{1}{2}})$ . The last column compares the range of smoothness classes over which optimal convergence rates have been established. For a more detailed explanation of the parameters $r,b>0$, see Assumption \ref{ass:source}, \ref{ass:dim}.   }\label{table3}
\begin{center}
\begin{tabular}{|l|l|l|l|}
\hline References &  $M$& Method & Smoothness degree\\
\hline \hline \mycite{NIPS2007_013a006f}& $O\left(n\right)$ & KRR with rvk & $r \in [0.5,1]$\\
\hline \mycite{features}& $O\left(\sqrt{n}\log(n)\right)$& KRR with rvk & $r \in [0.5,1]$\\
\hline \mycite{lanthaler2023error}& $O\left(\sqrt n\right)$ & KRR  with vvk & $r =0.5$\\
\hline 
\hline Our work & $O\left(\sqrt{n}\log(n)\right)$& Spectral Methods with vvk & $2r+b>1$\\
\hline 
\end{tabular}
\end{center}
\end{table}
The rest of the paper is organized as follows. In Section 2, we present our setting and review relevant results
on learning with kernels, and learning with random features. In Section 3, we
present and discuss our main results, while proofs are deferred to the appendix. Finally,
numerical experiments are presented in Section 4.

%%%%%%%%%%%%%%%%%%%%%%%%%%%%%%%%%%%%%%%%%%%%%%%%%%%%%%%%%%%%%%%%%%%%%%%%%%%%%%%%%%%%%%%

%\noindent 
{\bf Notation.} 
By $\cL(\cH_1, \cH_2)$ we denote the space of bounded linear operators between real separable Hilbert space $\cH_1$, $\cH_2$. 
We write $\cL(\cH, \cH) = \cL(\cH)$. For $\Gamma \in \cL(\cH)$, we denote by $\Gamma^*$ the adjoint operator and for compact $\Gamma$ 
by $(\lam_j(\Gamma))_{j}$ the  sequence of eigenvalues. If $h \in \cH$, we write $h \otimes h := \langle \cdot, h \ra h$. 
We let $[n]=\{1,...,n\}$.  For two positive sequences $(a_n)_n$, $(b_n)_n$ we write $a_n \lesssim b_n$ if $a_n \leq c b_n$
%, $a_n \gtrsim b_n$ if $a_n \geq c b_n$ 
for some $c>0$ and $a_n \simeq b_n$ if both $a_n \lesssim b_n$ and $b_n \lesssim a_n$. With $\|f\|^2_{L^2(\rho_x)}:=\int \|f(x)\|_\mathcal{Y}^2d\rho_x(x)$ we denote the $L^2$ norm and for $\mathbf{y}=(y_1,\dots y_n)\in\mathcal{Y}^{n}$ we let $\|\mathbf{y}\|_2^2:=\sum_{i=1}^n \|y_i\|_\mathcal{Y}^2$.

%%%%%%%%%%%%%%%%%%%%%%%%%%%%%%%%%%%%%%%%%%%%%%%%%%%%%%%%%%%%%%%%%%%%%%%%%%%%%%%%%%%%%%%
%%%%%%%% SECTION 2
%%%%%%%%%%%%%%%%%%%%%%%%%%%%%%%%%%%%%%%%%%%%%%%%%%%%%%%%%%%%%%%%%%%%%%%%%%%%%%%%%%%%%%%

\section{Setup}%Problem Setting
\label{sec:setting}

We let $\mathcal{X}$ be the input space and $\mathcal{Y}$ be the output space and assume that $\mathcal{X}$ denotes a Banach space and $\mathcal{Y}$ a separable Hilbert spaces with norms $\|.\|_\mathcal{X}$ and $\|.\|_\mathcal{Y}$, respectively. We impose these assumptions on the input and output spaces to facilitate the application of the theory of vector-valued kernels as developed in \cite{vvk1,vvk2}. The unknown data distribution on the data space $\cZ=\mathcal{X} \times \mathcal{Y}$ is denoted by $\rho$ while the marginal distribution on $\mathcal{X}$ is denoted as $\rho_X$ and the regular conditional distribution on $\mathcal{Y}$ given $x \in \mathcal{X}$ is denoted by $\rho(\cdot | x)$, see e.g. \cite{Shao_2003_book}. 

Given a measurable function $g: \mathcal{X} \to \mathcal{Y}$ we further define the expected risk as 
\begin{equation}
\label{eq:expected-risk}
\cE(g) := \mbe[ \ell (g(X), Y) ]\;,
\end{equation}  
where the expectation is taken w.r.t. the distribution $\rho$ and $\ell: \mathcal{Y} \times \mathcal{Y} \to \mbr_+$ is the least-square loss 
$\ell(t, y)=\frac{1}{2}\|t-y\|_\mathcal{Y}^2$. It is known that the global minimizer  of $\cE$ over the set of all measurable functions is 
given by the regression function $g_\rho(x)= \int_{\mathcal{Y}} y \rho(dy|x)$.

\subsection{Motivation of Kernel Methods with RFA}

First, we recall the most important definitions of vector-valued kernels. For more insights on vector valued kernels and their RKHSs we refer to  \cite{vvk1,vvk2} .

\begin{definition}
\label{def-vvk}
 
\begin{itemize}
\item []
\item Given a topological space $\mathcal{X}$ and a separable Hilbert space $\mathcal{Y}$, a map $K: \mathcal{X} \times \mathcal{X} \longrightarrow \mathcal{L}(\mathcal{Y})$ is
called a $\mathcal{Y}$-reproducing kernel on $\mathcal{X}$ if
$$
\sum_{i, j=1}^N\left\langle K\left(x_i, x_j\right) y_j, y_i\right\rangle_{\mathcal{Y}} \geq 0
$$
for any $x_1, \ldots, x_N$ in $\mathcal{X}, y_1, \ldots, y_N$ in $\mathcal{Y}$ and $N \geq 1$.

\item Given $x \in \mathcal{X}, K_x: \mathcal{Y} \rightarrow$ $\mathcal{F}(\mathcal{X} ; \mathcal{Y})$ denotes the linear operator whose action on a vector $y \in \mathcal{Y}$ is the function $K_x y \in \mathcal{F}(\mathcal{X},\mathcal{Y})$ defined by
$$
\left(K_x y\right)(t)=K(t, x) y \quad t \in \mathcal{X} .
$$

Given a $\mathcal{Y}$-reproducing kernel $K$, there is a unique Hilbert space $\mathcal{H}_K \subset \mathcal{F}(\mathcal{X} ; \mathcal{Y})$ satisfying
$$
\begin{array}{ll}
K_x \in \mathcal{L}\left(\mathcal{Y}, \mathcal{H}_K\right) & \forall \,\,\,u \in \mathcal{X} \\
F(x)=K_x^* F & \forall \,\,\,u \in \mathcal{X}, F \in \mathcal{H}_K,
\end{array}
$$
where $K_x^*: \mathcal{H}_K \rightarrow \mathcal{Y}$ is the adjoint of $K_x$. Note that from the second property we have that $K(x,t)=K^*_xK_t$. The space $\mathcal{H}_K$ is called the reproducing kernel Hilbert space associated with $K$, given by
$$
 \mathcal{H}_K=\overline{\operatorname{span}}\left\{K_x y \mid x \in \mathcal{X}, y \in \mathcal{Y}\right\} .
$$
\item A reproducing kernel $K: \mathcal{X} \times\mathcal{X} \rightarrow \mathcal{L}(\mathcal{Y})$ is called Mercer provided that $\mathcal{H}_K$ is a subspace of $\mathcal{C}(\mathcal{X}; \mathcal{Y})$.

\end{itemize}
\end{definition}

Now we can define kernel methods. These methods are nonparametric approaches defined by a kernel $K$ and a so called regularization function $\phi_\lambda$. The estimator then has the form
\begin{align}
    f_{\lambda}:=\phi_{\lambda}\left(\widehat{\Sigma}\right) \widehat{\mathcal{S}}^* \mathbf{y}, \label{kernel Method}
\end{align}
where  $\widehat{\mathcal{S}}^* \mathbf{y}:= \frac{1}{n}\sum_{i=1}^n  K_{x_i}y_i,\,\,$  $ \widehat{\Sigma}:=\frac{1}{n} \sum_{j=1}^{n} K_{x_{j}} K^*_{x_j}\,$ and $\mathcal{H}_K$ denotes the reproducing kernel Hilbert space (RKHS) of $K$. \cite{Muecke2017op.rates} established optimal rates for $\mathbb{R}$- reproducing kernel methods of the above form. The idea of this estimator is, when the sample size $n$ is large, the function $\widehat{\mathcal{S}}^* \mathbf{y}= \frac{1}{n}\sum_{i=1}^n  K_{x_i}y_i\in \mathcal{H}$ is a good approximation of its mean $\Sigma g_{\rho}=\int_{\mathcal{X}}  K_{x}g_{\rho}(x) d \rho_{X}$. Hence the spectral algorithm \eqref{kernel Method} produces a good estimator $f_{\lambda}$, if $\phi_{\lambda}\left(\widehat{\Sigma}\right)$ is an approximate inverse of $\Sigma$.
To motivate RFA we now consider the following examples. \vspace{0.2cm}\\
\textbf{Example 1 (KRR):} The probably most common example for explicit regularization is KRR for an $\mathbb{R}$- reproducing kernel $K$:
\begin{align}
f_\lambda(x)=\sum_{i=1}^n \alpha_i K\left(x_i, x\right), \quad \alpha=(\mathbf{K}+\lambda n I)^{-1} y, \label{KRR}
\end{align}
where $\mathbf{K}$ denotes the kernel gram matrix $\mathbf{K}_{i,j}=K(x_i,x_j)$. Note that this estimator can be obtained from \eqref{kernel Method} by choosing $\phi_\lambda(t)=\frac{1}{t+\lambda}$ \mycitep{Muecke2017op.rates}. In the above formula \eqref{KRR} the estimator has computational costs of order $O(n^3)$ since we need to calculate the inverse of an $n$ by $n$ matrix. However, if we assume to have a inner product kernel $K_M(x,\tilde{x})=\Phi_M(x)^\top \Phi_M(\tilde{x})$, where $\Phi_M$ is a feature map of dimension $M$, the computational costs can be reduced to $O(nM^2+M^3)$ \mycitep{features}.\vspace{0.2cm}\\
\textbf{Example 2 (HB):} To also give an example of implicit regularization we here analyse an acceleration method, namely the Heavyball method which can also be derived from \eqref{kernel Method} \cite{pagliana2019implicit} and is closely related to the gradient descent algorithm but has an additional momentum term:
\begin{align}
f_{k+1} &= f_k - \frac{\alpha}{n}\sum_{j=1}^n K(x_j , \cdot) (f_k(x_j ) - y_j)  + \beta( f_k - f_{k-1}) \;, \label{HB}
\end{align}
where $\alpha>0,\beta \geq 0$ describe the step-sizes. So in each iteration we have to update our estimator $f_k(x_j)$ for all data points. This results in a computational cost of order $O(kn^2)$.  However if we again assume to have a inner product kernel $K_M(x,x')=\Phi_M(x)^\top \Phi_M(x')$ we can  use theory of RKHS. Recall that the RKHS of $K_M$ can be expressed as 
$$\mathcal{H}_M=\{h:\mathcal{X}\rightarrow\mathbb{R}| \,\,\exists \,\theta\in\mathbb{R}^M\,\,\,s.t.\,\,\, h(x)= \Phi_M(x)^\top \theta\}$$ 
(see for example \cite{Ingo}). Since $K_M \in \mathcal{H}_M$ and therefore all iterations $f_k \in \mathcal{H}_M $, there exists some $\theta_k \in \mathbb{R}^M$ such that $f_k(x)=\Phi_M(x)^\top \theta_k$.  This implies that instead of running \eqref{HB} it is enough to update only the parameter vector:
\begin{align}
\heavy_{k+1} &= \heavy_k - \frac{\alpha}{n}\sum_{j=1}^n (\Phi_M(x_j)^\top \theta_k - y_j)  \Phi_M(x_j) + \beta( \heavy_k - \heavy_{k-1})\;. \label{paramGD}
\end{align}
The computational cost of the above algorithm \eqref{HB} is therefore reduced from $O(kn^2)$ to $O(knM)$. \vspace{0.2cm}\\
The basic idea of RFA is now to consider kernels which can be approximated by an tensor product \mycitep{NIPS2007_013a006f}:

\begin{align}
K_\infty(x,\tilde{x})\approx K_M(x,\tilde{x}):=\sum_{i=1}^p \frac{1}{M} \sum_{m=1}^M \varphi^{(i)}(x,\omega_m) \otimes \varphi^{(i)} (\tilde{x},\omega_m), 
\end{align}

where $\varphi^{(i)}: \mathcal{X} \times \Omega \rightarrow \mathcal{Y}$ , with some probability space $(\Omega,\pi)$. More precisely, this paper investigates RFA for kernels $K$ which have an integral representation of the form 

\begin{align}
K_\infty(x,\tilde{x})=\sum_{i=1}^p \int_\Omega \varphi^{(i)}(x,\omega) \otimes \varphi^{(i)}(\tilde{x},\omega) d\pi(\omega). \label{kernel}
\end{align}
In the case where $\mathcal{Y}=\mathbb{R}$ we have 

\begin{align}\label{kernelapprox}
 K_\infty(x,\tilde{x})&=\sum_{i=1}^p \int_\Omega \varphi^{(i)}(x,\omega)  \varphi^{(i)}(\tilde{x}\omega) d\pi(\omega) \\ 
 &\approx K_M(x,\tilde{x}):=\sum_{i=1}^p \Phi_M^{(i)}(x)^\top \Phi^{(i)}_M(\tilde{x}), 
\end{align}

where $\Phi_M^{(i)}: \mathcal{X} \rightarrow \mathbb{R}^M$ ,  $\Phi_M^{(i)}(x)=M^{-1/2}(\varphi^{(i)}(x,\omega_1), \dots, \varphi^{(i)}(x,\omega_M))$ is called a finite dimensional feature map and $\varphi^{(i)}:\mathcal{X}\times\Omega\rightarrow \mathbb{R}$.

Note that there are a large variety of standard kernels of the form \eqref{kernelapprox}. For example, the Linear kernel, the Gaussian kernel \mycitep{features} or  Tangent kernels \mycitep{domingos2020model}. In contrast to \cite{features}, we added an additional sum over different features $\varphi^{(i)}$ for a more general setting and to cover a special case of Tangent kernels namely the Neural-Tangent Kernel  (NTK) \mycitep{jacot2018neural} which provided a better understanding of neural networks in recently published papers \mycitep{nguyen2023neurons, nitanda2020optimal, Li21, Munteanu22, Oymak}. For one "hidden layer", the NTK is defined as 
\begin{align}
K_\infty\left(x, \tilde{x}\right) \coloneqq \mathbb{E}_\omega\sigma\left(\omega^\top x\right) \sigma\left(\omega^\top \tilde{x}\right)+\tau^{2}\left(x^{\top} \tilde{x}+\gamma^{2}\right)\mathbb{E}_\omega \sigma^{\prime}\left(\omega^\top x\right) \sigma^{\prime}\left(\omega^\top \tilde{x}\right),\label{NTK}
\end{align}
where $\tau, \gamma \in \mathbb{R}$ and $\sigma$ defines the so called activation function. According to our setting, the NTK from above can be recovered from \eqref{kernelapprox} by setting $p=d+2$ where $d$ denotes the input dimension and 
 $\varphi^{(i)}(x,\omega)= \tau x^{(i)}\sigma'\left(\omega^\top x\right)$ for $i \in [d]$ and $\varphi^{(d+1)}(x,\omega)= \sigma\left(\omega^\top x\right)$, $\varphi^{(d+2)}(x,\omega)= \tau\gamma\sigma'\left(\omega^\top x\right).$\\

Furthermore, instead of just real valued kernels \mycite{features} we consider vector-valued and operator-valued kernels within our framework. This enables us to establish optimal convergence rates for shallow neural operators (NO), see \mycite{nguyen2024optimalconvergenceratesneural}. Similar as for normal NNs, the NTK for NOs, can be recovered from our setting \eqref{kernel}.

\subsection{Kernel-induced operators and spectral regularization functions}

In this subsection, we specify the mathematical background of regularized learning.  It essentially repeats the setting in \mycite{Muecke2017op.rates} and \mycite{features} in summarized form.
First we introduce kernel induced operators and then recall basic definitions of linear regularization methods based on spectral theory for self-adjoint linear operators. These are standard methods for finding stable
solutions for ill-posed inverse problems. Originally, these methods were developed in the
deterministic context \mycitep{engl1996regularization}. Later on, they have been applied to probabilistic problems in machine learning  \mycitep{Caponetto, Muecke2017op.rates}.

Recall that $\mathcal{H}_M$ denotes the RKHS of the kernel $K_M$ defined in \eqref{kernel}.  
We denote by $\mathcal{S}_M : \cH_M \hookrightarrow  L^2(\mathcal{X} , \rho_X)$ the inclusion of $\cH_M$ into $L^2(\mathcal{X} , \rho_X)$ for $M \in \mathbb{N}\cup \infty$.
The adjoint operator $\cS^{*}_M: L^{2}(\mathcal{X}, \rho_X) \longrightarrow \mathcal{H}_{M}$ is identified as
$$
\cS^{*}_M g=\int_{\mathcal{X}}  K_{M,x} g(x)\rho_X(d x).
$$
The covariance operator $\Sigma_M: \mathcal{H}_{M} \longrightarrow \mathcal{H}_{M}$ and the kernel integral operator $\mathcal{L}_M: L^2(\mathcal{X} , \rho_X) \to L^2(\mathcal{X} , \rho_X) $ are given by
\begin{align*}
   \Sigma_Mf&\coloneqq \cS^*_M\cS_M f = \int_{\mathcal{X}} K_{M,x}K_{M,x}^* f \rho_X(d x),\\ 
   \mathcal{L}_M f&\coloneqq \cS_M \cS^*_M f = \int_{\mathcal{X}}  K_{M,x} f(x) \rho_X(d x),
\end{align*}

which can be shown to be positive, self-adjoint, trace class (and hence is compact).
The empirical versions of these operators, corresponding formally to taking the empirical distribution of $\rho_X$ in the above formulas, are given by

\begin{center}
\begin{align*}
&\widehat{\cS}_{M}: \mathcal{H}_{M} \longrightarrow \mathcal{Y}^{n},  &&\left(\widehat{\cS}_{M} f\right)_{j}= K_{M,x_{j}}^*f, \\
&\widehat{\cS}_{M}^{*}: \mathcal{Y}^{n} \longrightarrow \mathcal{H}_{M}, && \widehat{\cS}_{M}^{*} \mathbf{y}=\frac{1}{n} \sum_{j=1}^{n}  K_{M,x_{j}}y_{j}, \\
&\widehat{\Sigma}_{M}:=\widehat{\cS}_{M}^{*} \widehat{\cS}_{M}: \mathcal{H}_{M} \longrightarrow \mathcal{H}_{M},&& \widehat{\Sigma}_{M}=\frac{1}{n} \sum_{j=1}^{n}K_{M,x_{j}}K_{M,x_{j}}^*.
\end{align*}
\end{center}
Further let the numbers $\mu_{j}$ are the positive eigenvalues of $\Sigma_\infty$ satisfying $0<\mu_{j+1} \leq \mu_{j}$ for all $j>0$ and $\mu_{j} \searrow 0$.

\begin{definition}[Regularization function] Let $\phi :(0,1]\times [0,1]\rightarrow\mathbb{R}$ be a function and write $\phi_\lambda=\phi(\lambda,.)$. The family $\{\phi_\lambda\}_\lambda$
is called regularisation function, if the following condition holds:

\begin{itemize}
    \item[(i)] There exists a constant $D<\infty$ such that for any $0<\lambda \leq 1$
\begin{align}
\sup _{0<t<1}|t\phi_{\lambda}(t)| \leq D . \label{def.phi}
\end{align}

    \item[(ii)]There exists a constant $E<\infty$ such that for any $0<\lambda \leq 1$
$$
\sup _{0 < t \leq 1}\left|\phi_{\lambda}(t)\right| \leq \frac{E}{\lambda}.
$$
    \item[(iii)]Defining the residual $r_{\lambda}(t):=1-\phi_{\lambda}(t) t$, there exists a constant $c_{0}<\infty$ such that for any $0<\lambda \leq 1$
\begin{align}
\sup _{0 < t \leq 1}\left|r_{\lambda}(t)\right| \leq c_{0}. \label{residual}
\end{align}

\end{itemize}
\end{definition}

It has been shown in \mycite{10.1162/neco.2008.05-07-517, Muecke2017op.rates} that attainable learning rates are essentially linked with the qualification of the regularization $\left\{\phi_{\lambda}\right\}_{\lambda}$, being the maximal $\nu$  such that for any $q\in[0,\nu]$ and for any $0<\lambda \leq 1$

\begin{align}
\sup _{0 < t \leq 1}\left|r_{\lambda}(t)\right| t^{q} \leq c_{q} \lambda^{q}, \label{c_r}
\end{align}

for some constant $c_{q}>0$. \\

In the following, we present some of the most important examples of spectral filtering methods. \vspace{0.1cm}\\
\textbf{Example 1 (Tikhonov Regularization):}
The choice $\phi_\lambda(t)=\frac{1}{t+\lambda}$ corresponds to {\it Tikhonov regularization} In this case we have $D=E=c_q=1$. The qualification of this method is $\nu=1$ \mycitep{Muecke2017op.rates}.
\vspace{0.1cm}\\
\textbf{Example 2 (GD):}
The Landweber Iteration (gradient descent algorithm with constant stepsize $\alpha \in[0,1/\kappa^2]$) is defined by

$$
\phi_{1/k}(t)=\sum_{j=0}^{k-1}(1-\alpha t)^j \text { with } k=1 / \lambda \in \mathbb{N} .
$$

We have $D=E=\alpha$ and $c_q=(q/\alpha)^q$. The qualification of this algorithm is $\nu=\infty$ \mycitep{pagliana2019implicit}.
\vspace{0.1cm}\\
\textbf{Example 3 (Heavy Ball):}
In \mycite{pagliana2019implicit} it was shown that the heavy-ball method \eqref{HB} defines a spectral  filtering method for $\lambda=1/k^2$ and with parameters $E=D=2$. Moreover there exist a positive constant $c_\nu<+\infty$ such that the method has qualification $\nu$ \mycitep{pagliana2019implicit}. An other example of an accelerated spectral filtering method is {\it Nesterov} \mycitep{pagliana2019implicit}.

%%%%%%%%%%%%%%%%%%%%%%%%%%%%%%%%%%%%%%%%%%%%%%%%%%%%%%%%%%%%%%%%%%%%%%%%%%%%%%%%%%%%%%%
%%%%%%%% SECTION 3
%%%%%%%%%%%%%%%%%%%%%%%%%%%%%%%%%%%%%%%%%%%%%%%%%%%%%%%%%%%%%%%%%%%%%%%%%%%%%%%%%%%%%%%

\section{Main Results}
\label{sec:main-results}

\subsection{Assumptions and Main Results}

In this section we formulate our assumptions and state our main results. 

\begin{assumption}[Data Distribution]
\label{ass:input}
 There exists positive constants $Q$ and $Z$ such that for all $l \geq 2$ with $l \in \mathbb{N}$,
$$
\int_{\mathcal{Y}}\|y\|_\mathcal{Y}^l d \rho(y \mid x) \leq \frac{1}{2} l ! Z^{l-2} Q^2
$$
$\rho_X$-almost surely.\end{assumption}
The above assumption is very standard in statistical learning theory. It is for example satisfied if $y$ is bounded almost surely. Obviously, this assumption implies that the regression function $g_\rho$ is bounded almost surely, as
$$
\left\|g_\rho(x)\right\|_\mathcal{Y} \leq \int_{\mathcal{Y}}\|y\|_\mathcal{Y} d \rho(y \mid x) \leq\left(\int_{\mathcal{Y}}\|y\|_{\mathcal{Y}}^2 d \rho(y \mid x)\right)^{\frac{1}{2}} \leq Q\,.
$$

\begin{assumption}[Kernel]
\label{ass:kernel}
Assume that the kernel $K_\infty$ has an integral representation of the form \eqref{kernel} with
$\sum_{i=1}^p\|\varphi^{(i)}(x,\omega)\|^2_\mathcal{Y}\leq \kappa^2$ almost surely. 
\end{assumption}
Note that this implies that $\|K(x,\tilde{x})\|_{HS}\leq\kappa^2$ almost surely.

\begin{assumption}[Source Condition]
\label{ass:source}
Let $R>0$, $r>0$. Denote by $\mathcal{L}_\infty:  L^2(\mathcal{X} , \rho_X)\to  L^2(\mathcal{X} , \rho_X)$ the kernel integral operator associated to $K_\infty$. We assume 
\begin{align}
g_\rho  = \mathcal{L}_\infty^r h \;, \label{hsource}
\end{align}
for some $h \in L^2(\mathcal{X} , \rho_X)$, satisfying $||h||_{L^2(\rho_X)} \leq R$ . 
\end{assumption} 
This assumption characterizes the hypothesis space and relates to the regularity of the regression function $g_\rho$. The bigger $r$ is, the smaller the hypothesis space is, the stronger the assumption is, and the easier the learning problem is, as $\mathcal{L}^{r_{1}}\left(L_{\rho_{X}}^{2}\right) \subseteq \mathcal{L}^{r_{2}}\left(L_{\rho_{X}}^{2}\right)$ if $r_{1} \geq r_{2}$.  
The next assumption relates to the capacity of the hypothesis space.
\begin{assumption}[Effective Dimension]
\label{ass:dim} For some $b \in[0,1]$ and $c_{b}>0, \mathcal{L}_\infty$ satisfies
\begin{align}
\mathcal{N}_{\mathcal{L}_{\infty}}(\lambda):=\operatorname{tr}\left(\mathcal{L}_\infty(\mathcal{L}_\infty+\lambda I)^{-1}\right) \leq c_{b} \lambda^{-b}, \quad \text { for all } \lambda>0 \label {effecDim}
\end{align}
and further we assume that $2r+b>1$.
\end{assumption}
The left hand-side of \eqref{effecDim} is called effective dimension or degrees of freedom \mycitep{Caponetto}. It is related to covering/entropy number conditions \mycitep{Ingo}. The condition \eqref{effecDim} is naturally satisfied with $b=1$, since $\Sigma$ is a trace class operator which implies that its eigenvalues $\left\{\mu_{i}\right\}_{i}$ satisfy $\mu_{i} \lesssim i^{-1}$. Moreover, if the eigenvalues of $\Sigma$ satisfy a polynomial decaying condition $\mu_{i} \sim i^{-c}$ for some $c>1$, or if $\Sigma$ is of finite rank, then the condition \eqref{effecDim} holds with $b=1 / c$, or with $b = 0$. The case $b = 1$ is refereed as the capacity independent case. A smaller $b$ allows deriving faster convergence rates for the studied algorithms. The assumption $2r+b>1$ refers to easy learning problems and if $2r+b\leq 1$ one speaks of hard learning problems \mycitep{pillaudvivien2018statistical}. In this paper we only investigate easy learning problems and leave the question, how many features $M$ are needed to obtain optimal rates in hard learning problems \mycitep{Lin_2020}, open for future work.
\\

We now derive a generalisation bound of the excess risk $\|g_\rho-\mathcal{S}_M f_\lambda^M\|_{L^2(\rho_x)}$ with respect to our RFA estimator,
\begin{align*}
   f_\lambda^M &\coloneqq \phi_\lambda(\widehat\Sigma_M) \widehat{\cS}_{M}^{*} y \,.
\end{align*}
The main idea of our proof is based on a bias-variance type decomposition: Further introducing

\begin{align*}
   f_\lambda^*&\coloneqq \mathcal{S}^*_M\phi_\lambda(\mathcal{L}_M) g_{\rho},
\end{align*}
we write 
\begin{align}
\|g_\rho-\mathcal{S}_M f_\lambda^M\|_{L^2(\rho_x)}&\leq \|g_\rho-\mathcal{S}_Mf_\lambda^*\|_{L^2(\rho_x)} + \|\mathcal{S}_Mf_\lambda^*-\mathcal{S}_M f_\lambda^M\|_{L^2(\rho_x)}\\[7pt]
&=: \text{ BIAS } + \text{ VARIANCE }. \label{excessrisk}
\end{align}

We bound the bias and variance part separately in Proposition \ref{mainprop2} and \ref{mainprop} to obtain the following theorem.

\begin{theorem}
\label{theo1}
Provided the Assumptions \ref{ass:input} ,\ref{ass:kernel} , \ref{ass:source}, \ref{ass:dim},  we have for 
$\lambda=C n^{-\frac{1}{2r+b}}\log^3(2/\delta)$ and $\delta\in(0,1)$ with probability at least $1-\delta$,  
\begin{align*}
\|g_\rho-\mathcal{S}_M f_\lambda^M\|_{L^2(\rho_x)}\leq \bar{C}n^{-\frac{r}{2r+b}}\log^{3r+1}\left(\frac{1}{\delta}\right)
\end{align*}
as long as $\nu \geq r\vee 1$,

\begin{align*}
M\geq \tilde{C}\log(n) \cdot \begin{cases}
n^{\frac{1}{2r+b}}& r\in\left(0,\frac{1}{2}\right)\\
n^{\frac{1+b(2r-1)}{2r+b}}  & r\in\left[\frac{1}{2},1\right] \\
n^{\frac{2r}{2r+b}} & r \in(1,\infty)\,\\
\end{cases}\,,\\
\end{align*}
and $n\geq n_0:= e^{\frac{2r+b}{2r+b-1}}$, where the constants $C,\tilde{C}$ and $\bar{C}$ are independent of $n,M,\lambda$ and can be found in section \ref{I}.
\end{theorem}
If we can not make any assumption on the effective dimension, i.e. assuming the worst case $b=1$ we obtain the following corollary. 

\begin{corollary}\label{cor:rates}
Provided the Assumptions \ref{ass:input} ,\ref{ass:kernel} , \ref{ass:source}, with $r=0.5$ we have for 
$\lambda=C n^{-\frac{1}{2}}\log^3(2/\delta)$ and $\delta\in(0,1)$ with probability at least $1-\delta$,  
\begin{align*}
\|g_\rho-\mathcal{S}_M f_\lambda^M\|_{L^2(\rho_x)}^2\leq  \bar{C}^2n^{-\frac{1}{2}}\log^{5}\left(\frac{1}{\delta}\right)
\end{align*}
as long as $\nu \geq r\vee 1$, $M\geq \tilde{C}\log(n) \cdot n^\frac{1}{2}$ .
\end{corollary}

\paragraph{Discussion.}

\begin{itemize}
\item Theorem \ref{theo1} establishes generalization rates for a broad class of spectral filtering methods, encompassing both explicit approaches such as Tikhonov regularization and implicit schemes including gradient descent and the Heavy–Ball method. In the well-specified setting, achieving an upper bound of $O(\frac{1}{\sqrt{n}})$ for the excess risk requires only $t=O(\sqrt{n})$ iterations and $M=O(\sqrt{n}\log(n))$ random features. For target functions with regularity $r\geq1$, attaining the optimal convergence rate necessitates $t=O(n^{\frac{1}{2r+b}})$ iterations and $M=O(t^{2r}\log(n))$ random features. This indicates that smoother target functions demand less iterations, but a greater number of random features to achieve optimal generalization.
\item The rate of convergence in Theorem \ref{theo1} is known to be minimax optimal 
in the RKHS framework \citep{Caponetto,Muecke2017op.rates}.
\item Compared to the work of \mycite{features,lanthaler2023error}, we establish convergence rates for all spectral filtering methods and extend their results for KRR by deriving optimal rates for smoothness classes $r<1/2$ satisfying $2r+b>1$ (the {\it easy learning regime}). 
\item Regarding the number of required random features, we need the same amount as \mycite{features}, namely $M=\sqrt{n}\log(n)$. The work \mycite{lanthaler2023error} relies on slightly different source assumptions, which, however, coincide with ours in the well-specified case. Using a "random kitchen sinks" approach from \mycite{NIPS2008_0efe3284}, they managed to remove the logarithmic factor, proving that $M=\sqrt{n}$ random feature suffices to achieve optimal rates. However, the work of \mycite{lanthaler2023error} does not incorporate any prior knowledge about the effective dimension and can therefore only establish optimal rates in the well specified case $b=1, r=0.5$.
\end{itemize}

\section{Application to Neural Operators}

\input{AppNOs}

\section{Numerical Illustration}
\label{sec:numerics}

We analyze the behavior of kernel GD (algorithm \eqref{paramGD} for $\beta=0$) with the RF of the NTK kernel \eqref{NTK}. 
In our simulations we used $n=5000$ training and test data points from a standard normal distributed data set with input 
dimension $d=1$ and a subset of the SUSY\footnote{ https://archive.ics.uci.edu/ml/datasets/SUSY} classification data set with 
input dimension $d=14$. The measures we show in the following simulation are an average over 50 repetitions of the algorithm.  
Our theoretical analysis suggests that only a number of RF of the order of $M = O(\sqrt{n}\cdot d)$ suffices to gain optimal 
learning properties\footnote{ The linear factor of $d$ is hidden in the constants of our results and can be found in the proof section.}. 
Indeed in Figure \ref{F1} we can observe for both data sets that over a certain threshold of the order $M = O(\sqrt{n}\cdot d)$ and fixed $T$, 
increasing the number of RF does not improve the test error of our algorithm. 

\begin{figure}[h]
\centering
\includegraphics[width=0.3\columnwidth, height=0.23\textheight]{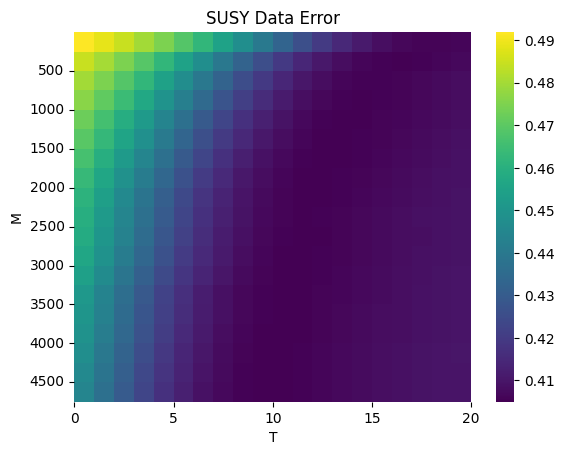}
\includegraphics[width=0.3\columnwidth, height=0.23\textheight]{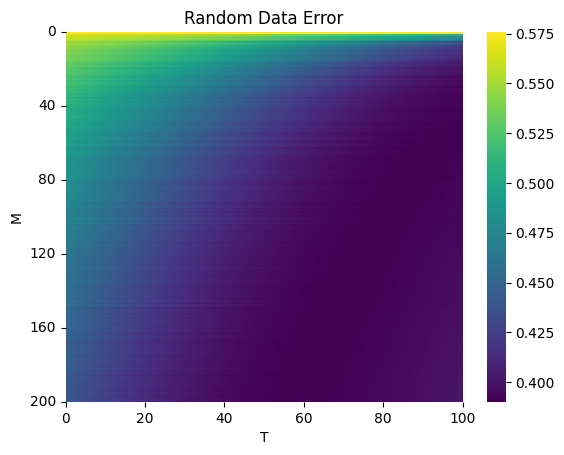}
\caption{Heat plot of the test-error for different numbers of RF $M$ and iterations $T$.\\ 
\label{F1}
{\bf Left:} Error of SUSY data set.
{\bf Right:} Error of random data set.}
\end{figure}

\section{Conclusion}

In this work we have unified two of the most influential acceleration paradigms for large‑scale non‑parametric learning, spectral regularization methods and random feature approximation, under a single operator‑theoretic lens. Our main theorem shows that, for any spectral algorithm, the excess risk of the random–feature estimator retains the optimal learning rate, provided that the number of random features scales only like $O(\sqrt{n}\log(n))$.

The analysis advances the state of the art in several directions:

\begin{itemize}
    \item \textbf{Generality.} Previous sharp results were limited to kernel‑ridge regression; we cover the full spectrum from explicit Tikhonov‑type schemes to implicit methods such as (accelerated) gradient descent and Heavy‑Ball, thereby closing the “algorithmic gap’’ between theory and practice for random‑feature models.
    \item \textbf{Beyond‑RKHS targets.} By formulating source conditions in terms of powers of the integral operator ... we allow target functions that lie \textit{outside} the RKHS but still satisfy $2r+b> 1$. This removes the saturation bottleneck inherent in prior RF analyses and yields minimax‑optimal rates for arbitrary smoothness $r>0$.
    \item \textbf{Vector‑ and operator‑valued kernels.} The framework seamlessly incorporates sums of integral kernels, including Neural Tangent and Neural Operator kernels that arise in modern deep architectures; the accompanying numerical study on NTK random features confirms the predicted $M=O(\sqrt{n})$ regime on both synthetic and real (SUSY) data.
\end{itemize}
Looking ahead, two questions remain open. First, determining the optimal feature complexity in the \textit{hard‑problem} regime $2r+b\leq 1$. Second, extending the theory to adaptive feature samplers or data‑dependent sketching and early stopping times not depending on the unknown smoothness parameters $r$ and $b$, could further reduce computational footprints. We hope the operator‑theoretic perspective offered here will serve as a blueprint for these and other advances in scalable kernel and neural‑operator learning.

\bibliography{bib_iteration}
%\bibliographystyle{alpha}
%\bibliography{bib_iteration}

%%%%%%%%%%%%%%%%%%%% APPENDIX %%%%%%%%%%%%%%%%%%%%%%%%%%%%%%%

\newpage

\appendix
\input{Beweis}

\input{T.U}

\end{document}

%% file: AppNOs.tex
In this section, we demonstrate how our results can be applied to derive convergence rates for neural operators in the NTK regime. To this end, we begin by recalling the definitions of a neural operator and the associated neural tangent kernel. We essentially repeat the setting from \mycite{nguyen2024optimalconvergenceratesneural} in summarized form. 
%%%%%%%%%%%%%%%%%%%%%%%%%%%%%%%%%%%%%%%%%%%%%%%%%%%%%%%%%%%%%%%%%%%%%%%%%%%%%%%%%%%%%%%%%%%%%%%%%%%%%%%%%%%%%%%%%%%%

\subsection{Two-Layer Neural Operator}

We start by redefining the in and output space as follows. 
We let  $\mathcal{X} \subseteq \mathbb{R}^{d_x}$ be the input space of our function spaces.  The unknown data 
probability measure on the input data space $\cX$ is denoted by $\mu$.
Further we let $\mathcal{U}$ be the function input space, mapping from $\cX \to \cY \subset \mathbb{R}^{d_y}$ and $\mathcal{V}$ be the function target space, mapping from $\mathcal{X} \to \tilde{\mathcal{Y}} \subset \mathbb{R}$.  The unknown data distribution on the data space $\mathcal{U} \times \mathcal{V}$ is denoted by $\rho$, while the marginal distribution on $\mathcal{U}$ is denoted as $\rho_{u}$ and the regular conditional distribution on $\mathcal{V}$ given $u \in \mathcal{U}$ is denoted by $\rho(\cdot | u)$. Here,  the distributions $\rho$, $\rho_u$ and $\mu_x$ are known only through $\nU$ i.i.d. 
samples $((u_1, v_1), ..., (u_{n_{\Scale[0.4]{\mathcal{U}}}}, v_{n_{\Scale[0.4]{\mathcal{U}}}})) \in (\cU \times \cV)^{\nU}$, 
evaluated at $\nX$ points $(x_1, \dots x_{\nX})\in\mathcal{X}^{\nX}$.

The hypothesis class considered is given by the following set of two-layer neural operators: 
Let $M \in \mathbb{N}$ be the network width. 
Given an activation function $\sigma: \mathbb{R} \to \mathbb{R}$ acting point wise and some continuous 
operator $A:\mathcal{U}\rightarrow \mathcal{F}(\mathcal{X},\mathbb{R}^{d_k}) $  we consider the class 

\begin{align}
\label{Oclass}
\mathcal{F}_{M} &:=\left\{ G_\theta :\mathcal{U} \to \mathcal{V}  \mid \; G_\theta(u)(x) =  
\frac{1}{\sqrt M} \left\langle a, \sigma\left( B_1 A(u)(x) + B_2u(x)+B_3c(x) \right)\right \rangle\;, \right. \nonumber \\
& \quad \left. \theta =(a, B_1,B_2,B_3 ) \in \mathbb{R}^M \times \mathbb{R}^{ M \times d_k}\times \mathbb{R}^{ M \times d_y}  \times \mathbb{R}^{ M \times d_b}\right\}\;, 
\end{align}

where we denote with $\langle.,.\rangle$ the euclidean vector inner product. 
We condense all parameters in $\theta = (a, B_1,B_2,B_3 )=(a, B )\in \Theta$, with 
$B=(B_1,B_2,B_3)\in\mathbb{R}^{M\times\tilde{d}}$, $\tilde{d}:=d_k+d_y+d_b$ and equip the parameter space $\Theta$ with the euclidean vector norm 
\begin{align}
\|\theta \|_\Theta^2 = ||a||_2^2 + ||B||_F^2   = \|a\|_2^2 + \sum_{m=1}^M\|b_m\|_2^2 ,   \label{tnorm}
\end{align}

for any $\theta \in \Theta$, where we used the notation $B=(b_1,\dots,b_M)^\top\in\mathbb{R}^{M\times\tilde{d}}$ .

\vspace{0.2cm}
For the activation function $\sigma^2$, we impose the following assumption.

\vspace{0.2cm}

\begin{assumption}[Activation Function]
\label{ass:neurons} 
There exists  $C_\sigma>0$  such that $\left\|\sigma^{\prime \prime}\right\|_{\infty} \leq C_\sigma,\left\|\sigma^{\prime}\right\|_{\infty} \leq C_\sigma$ , $|\sigma(u)| \leq 1+|u| \text { for } \forall u \in \mathbb{R}$.
In addition we assume that the second derivative $\sigma''$ is Lipschitz continuous.  
\end{assumption}

\vspace{0.2cm}

In our operator class \ref{Oclass}, we first apply an operator transformation $B_1A(u)$, where $B_1$ 
defines a matrix and typically, $A$ defines a kernel integral operator. Depending on the choice of $A$, our 
operator class covers a wide range of practically well-studied neural operators, such as Fourier Neural Operators, Graph Neural Operators, and Low-rank Neural Operators \mycitep{Kovachki2023NeuralOL, huang2024operator, kovachki2024operator}. The kernel integral 
operation $A(u)$, can be seen as an intuitive extension of matrix multiplication in traditional neural 
networks to an infinite-dimensional setting. Inspired by Residual Neural Networks, the second 
transformation, $B_2 u(\,.)$ with $B_2\in\mathbb{R}^{M\times d_y}$ is just a matrix multiplication 
and keeps information of the original input $u(\,.)$. The final component, $B_3 c(\,.)$ with 
$c: \mathcal{X}\to \mathbb{R}^{d_b}$ defines the bias. Unlike standard neural networks, the bias 
here is typically not a constant but a function. This function can be a linear transformation or a shallow 
neural network with trainable parameters. However, to keep our proofs succinct, we assume that $c$ is given 
and instead multiply the bias with a trainable matrix $B_3\in\mathbb{R}^{M\times d_b}$, see Figure \ref{op:class}. 

\vspace{0.3cm}

\begin{figure}[t]
    \centering
    \includegraphics[width=0.5\linewidth]{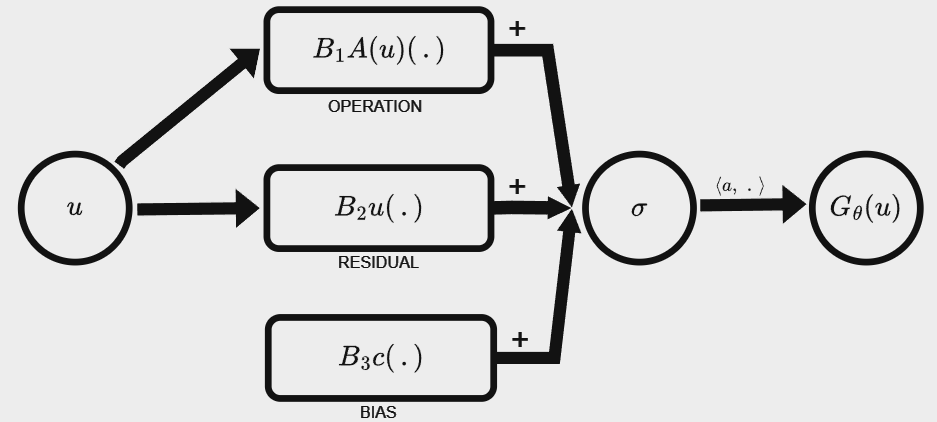}
\caption{Depiction of the architecture of our operator class.}
\label{op:class}    
\end{figure}

\vspace{0.3cm}

{\bf Gradient Descent.} We aim at analyzing the generalization properties of gradient descent, 
whose basic iterations are given by   
\begin{align}
\heavy_{t+1}^j &= \heavy_t^j - \alpha \partial_{\theta^j} \hat \cE(G_{\theta_t}) \nonumber  \\
&= \heavy_t^j - \frac{\alpha}{n_{\Scale[0.4]{\mathcal{U}}}}\sum_{i=1}^{n_{\Scale[0.4]{\mathcal{U}}}} 
\left \langle (G_{\theta_t}(u_i) - v_i) , \partial_{\theta^j} G_{\theta_t}(u_i)\right\rangle_{\nX},  
\label{GDalgor}
\end{align}
with $\alpha > 0$ being the stepsize, for some initialization $\theta_0 \in \Theta$ and   $\langle u,u'\rangle_{\nX}:=\frac{1}{\nX}\sum_{i=1}^{\nX}u(x_i)u'(x_i)$ denotes the empirical inner product.

%%%%%%%%%%%%%%%%%%%%%%%%%%%%%%%%%%%%%%%%%%%%%%%%%%%%%%%%%%%%%%%%%%%%%%%%%%%%%%%%%%%%%%%%%%%%%%%%%%%%%%%%%%%%%%%%%%%%

\subsection{Neural Tangent Kernel}

For standard fully connected neural networks, the feature map of the NTK 
is defined via the gradient of the neural network at initialization. Similar we can define a feature map for neural operators by $\Phi^M: \cU \to \mathcal{F}(L^2(\mathcal{X},\mu_x),\Theta);\, \Phi^M(u) := \Phi^M_u,$ 

\[  \Phi_u^M(v) := \nabla_\theta\left\langle G_{\theta_0}(u),\,v\right\rangle_{L^2(\mu_x)} \;. \]

%%%%%%%%%%%%%%%%%%%%%%%%%%%%%%%%%%%%%%%%%%%%%%%%%%%%%%%%%%%%%%%%%%%%%%%%%%%%%%%%%%%%%%%%%%%%%%%%%%%%%%%%%%%%%%%%%%%%

This map defines a $L^2(\mathcal{X},\mu_x)$-reproducing kernel in the sense of Definition \ref{def-vvk}, for 
any $u , u' \in \cU$ via 
\begin{align*}
 K_M(u, u') &=  (\Phi^M_u)^*  \Phi^M_{u'} = \sum_{p=1}^{P} \partial_pG_{\theta_{0}}(u)\otimes\partial_p G_{\theta_{0}}(u')\;\\
&=\frac{1}{M} \sum_{m=1}^{M} \sigma\left(\left\langle b_{m}^{(0) }, J(u)\right\rangle \right) \otimes\sigma\left(\left\langle b_{m}^{(0) }, J(u')\right\rangle\right)\,+\\
&\,\,\,\,\,\,\,\,\,\frac{\tau^2}{M} \sum_{m=1}^{M} \sum_{j=1}^{\tilde{d}}\sigma^{\prime}\left(\left\langle b_{m}^{(0) }, J(u)\right\rangle \right)J(u)^{(j)}\otimes \sigma^{\prime}\left(\left\langle b_{m}^{(0)}, J(u')\right \rangle \right)J(u')^{(j)},
\end{align*}
where we used for the parameter matrix at initialization the notations $a^{(0)}=(a_1^{(0)},\dots,a_M^{(0)})$ with $|a_j^{(0)}|=\tau,\,\,\forall j=1,\dots,M$, $\,B^{(0)}=(b_1^{(0)},\dots,b_M^{(0)})^\top\in\mathbb{R}^{M\times\tilde{d}}$ . $P=M\cdot(\tilde{d}+1)$ denotes the amount of parameters and for any $u\in\mathcal{U}$ we set $J(u)(x):=(A(u)(x),u(x),c(x))^\top$. Further we defined the tensor product with respect to the $L^2(\rho_x)$ scalar product, i.e. $(u\otimes v)(f):= u\cdot\langle v\,,\,f\rangle_{L^2_{\rho_x}}$ . The corresponding RKHS is given by 
\[ \mathcal{H}_{M}=
\left\{ H \in\, \mathcal{F}( L^2(\mathcal{X},\mu_x); L^2(\mathcal{X},\mu_x))\,\mid \,H (u)\,= \, 
\langle \tilde{\theta},  \nabla_{\theta} G_{\theta_{0}}(u)\tilde{\theta}\rangle,  \tilde{\theta} \in \Theta\right\},  \]
(\mycite{vvk2} Example 4). 

%%%%%%%%%%%%%%%%%%%%%%%%%%%%%%%%%%%%%%%%%%%%%%%%%%%%%%%%

From \mycite{nguyen2024optimalconvergenceratesneural} Proposition 2.3 we know that for all 
$u, u'\in\mathcal{U}$ ,  $K_M(u,u')$ converges for $M \to \infty$ to the limiting neural tangent kernel
$K_\infty(u,u')\in\mathcal{C}(L^2(\mu_x),L^2(\mu_x))$  in Hilbert-Schmidt norm, where
\begin{align*}
K_\infty(u,u') &=\mathbb{E}_{\theta_{0}} \sigma\left(\left\langle b^{(0) }, 
J(u)\right\rangle \right) \otimes\sigma\left(\left\langle b^{(0) }, J(u')\right\rangle\right)\,+\\
&\,\,\,\,\,\,\,\,\,\tau^2\sum_{j=1}^{\tilde{d}}
\mathbb{E}_{\theta_{0}}\sigma^{\prime}\left(\left\langle b^{(0) }, J(u) 
\right\rangle \right)J(u)^{(j)}\otimes \sigma^{\prime}\left(\left\langle b^{(0)}, J(u')\right \rangle 
\right)J(u')^{(j)} \,.
\end{align*}

\subsection{Generalization Bound of NOs}

We aim for bounding the excess risk of our Neural Operator i.e. $\|G_{\theta_T} - g_\rho\|_{L^2_{\rho_u}}$, where $g_\rho$ denotes the regression operator, defined as 
\begin{align}
g_\rho(u):= \int_{\mathcal{V}} v(.) \rho(dv|u)  \label{targetfct}
\end{align}
and where the $L^2_{\rho_u}$ norm is defined as $\|G \|_{L^2_{\rho_u}} ^2:= \int_{\mathcal{U}} \|G(u)\|_{L^2_{\mu_x}}^2d \mu_x $.

The proof is based on a suitable error decomposition. To this end, 
we further introduce a linear Taylor approximation of the neural Operator around $\theta_{0}$ in $\cH_M$, $H_t(u)(x) :=  \left\langle\nabla (G_{\theta_{0}}(u)(x)), \theta_{t}-\theta_{0}\right\rangle_{\Theta},$ and define with $F_t^M \in \cH_M$ our kernel estimator with respect to the gradient descent spectral filtering function and the empirical vvNTK, $K_M$. Now we can decompose our excess risk as,

\begin{align}
\label{errordecomp2}
\|G_{\theta_T} - g_\rho\|_{L^2_{\mu_u}}
&\leq \underbrace{ \|G_{\theta_T}  - \mathcal{S}_M H_T\|_{L^2_{\rho_u}}}_{\mathcal{I}} + 
\underbrace{  \|\mathcal{S}_M (H_T - F_T^M )\|_{L^2_{\rho_u}}}_{\mathcal{II}}  + 
\underbrace{  \|\mathcal{S}_MF_T^M-  g_\rho\|_{L^2_{\rho_u}}}_{\mathcal{III}}  \;.
\end{align}
where $\mathcal{S}_M : \cH_M \hookrightarrow L^2(\cU , \mu_u)$ is the inclusion of $\cH_M$ into $L^2(\cU , \rho_u)$.

\vspace{0.3cm}

The first error $\mathcal{I}$ in \eqref{errordecomp2} describes an Taylor approximation error.  More precisely we use a Taylor expansion in $\theta_{t}$ around the initialization 
$\theta_0$. We choose the initialization $\theta_0$ such that $G_{\theta_{0}}=0$. Therefore, for any $x \in \mathcal{X}$ and $t \in [T]$, we have
\begin{align}
\label{eq:taylor}
G_{\theta_{t}}(u)(x)&= G_{\theta_{0}}(u)(x) + \mathcal{S}_M\left\langle\nabla (G_{\theta_{0}}(u)(x)), \theta_{t}-\theta_{0}\right\rangle_{\Theta} 
+ r_{(\theta_{t},\theta_{0})}(u)(x) \nonumber \\
&= \mathcal{S}_M H_t(u)(x)+r_{(\theta_{t},\theta_{0})}(u)(x) \; .  
\end{align} 
Here, $r_{(\theta_{t},\theta_{0})}(x)$ denotes the  Taylor remainder and can 
be uniformly bounded by 
\[\|G_{\theta_T}  - \mathcal{S}_M H_T\|_{L^2_{\rho_u}}\leq \| r_{(\theta_{t},\theta_{0})} \|_\infty \lesssim 
\; \frac{\|\theta_t-\theta_0\|_\Theta^3}{\sqrt{M}} \;, \] 
as  \mycite{nguyen2024optimalconvergenceratesneural} Proposition D.2 shows.

\vspace{0.3cm}

The second error term $\mathcal{II}$ in \eqref{errordecomp2} depends on the number of neurons $M$ and on the number of second stage samples $\nX$, see \mycite{nguyen2024optimalconvergenceratesneural} Theorem B.4. 
More precisely, we obtained 
\[  \|\mathcal{S}_M (H_T - F_T^M )\|_{L^2_{\rho_u}}  \lesssim \log(T)\|\theta_t-\theta_0\|_\Theta^3\left(\frac{1}{\sqrt{M}}+\frac{1}{\sqrt{\nX}}\right)\;,\] 
\vspace{0.3cm}

The last error term $\mathcal{III}$ \eqref{errordecomp2} describes the generalization error of KGD. This error was bounded in Theorem \ref{theo1} for $T\approx \nU^{\frac{1}{2r+b}}$  by 
\[ \|\mathcal{S}_MF_T^M-  g_\rho\|_{L^2_{\rho_u}}\lesssim \nU^{-\frac{r}{2r+b}} \;. \]

Hence, we obtain that the excess risk of the neural operator is bounded by the excess risk of the KGD estimator, provided that both the number of second-stage samples and the number of neurons are sufficiently large. More precisely, we obtain the following theorem. 

\begin{theorem}[\mycite{nguyen2024optimalconvergenceratesneural} (Theorem 3.5.)]
\label{maintheo}
Suppose Assumptions \ref{ass:neurons}, \ref{ass:input},   \ref{ass:source} and \ref{ass:dim} are satisfied. 
Assume further that $\alpha\in (0,\kappa^{-2})$ , $2r + b >1$, $\nU\geq n_0:= e^{\frac{2r+b}{2r+b-1}}$, 
\begin{align*}
T = C\nU^{\frac{1}{2r+b}}, \quad M\geq \tilde{C} B_T^6\log^2(\nU) T^{2r\vee1}\,,\quad \nX\geq \tilde{C}B_T^2 T^{2r} \log^2(T),
\end{align*}
and
\begin{equation}\label{weights}
   \forall \;\; t \in [T]\;:\; \;\; \| \theta_t -\theta_0 \|_\Theta \leq B_T \;,  
\end{equation}   
for some $B_T>0$. Then we have with probability at least $1-\delta$,\vspace{0.1cm}
\begin{align}
\label{eq:final-bound}
\| G_{\theta_T}- g_\rho\|_{L^2(\rho_u)}
&\leq \;  \bar{C} \; \nU^{-\frac{r}{2r+b}}\; \log^3(2/\delta ) \;,
\end{align}
\vspace{0.2cm}
with $C, \tilde{C}, \bar{C}>0$ independent of $\nU,\nX, M, T,B_T$.
\end{theorem}

\paragraph{Discussion.}
In \mycite{nguyen2024optimalconvergenceratesneural} Theorem 3.7, it was shown that Assumption \ref{weights} is satisfied for $B_T\leq O(\log(T))$ with high probability. The above theorem thus implies that a two-layer neural operator achieves the optimal convergence rate. Notably, the number of neurons required is comparable to the number of random features needed to control the excess risk of the KGD estimator.

%% file: Beweis.tex
\section{Appendix}
The proof section is organized as follows. In Appendix I we give the proofs of our main results, in Appendix II we prove some technical inequalities and Appendix III contains all the needed concentration inequalities. 

For the proofs we will use the following shortcut notations. For any Operator $A$ and $\lambda>0$ we set $A_\lambda:= A+\lambda I$ where $I$ denotes the identity operator and for any function $g$ we define the vector $\bar{g}=(g(x_1), \dots , g(x_n)) \in\mathbb{R}^n$.
 
\subsection{Appendix I}
\label{I}

First we start bounding the bias part of our excess risk \eqref{excessrisk}. 

\begin{proposition}
\label{mainprop2}
Given the Assumptions \ref{ass:input} ,\ref{ass:kernel} , \ref{ass:source}, \ref{ass:dim}, then we have for
\begin{align*}
M&\geq 
\begin{cases}
\frac{8 p\kappa^2 \beta_\infty}{\lambda}\vee C_{\delta,\kappa} & r\in\left(0,\frac{1}{2}\right)\\
\frac{(8 p\kappa^2 \beta_\infty)\vee C_1^{\frac{1}{r}}}{\lambda}\vee \frac{C_2}{\lambda^{1+b(2r-1)} } \vee C_{\delta,\kappa} & r\in\left[\frac{1}{2},1\right] \\
\frac{C_3}{\lambda^{2r}}\vee C_{\delta,\kappa} & r \in(1,\infty),\\
\end{cases}
\end{align*}
where  $C_1=2(4\kappa\log\frac{2}{\delta})^{2r-1}(8p\kappa^2\beta_\infty)^{1-r}$ ,  $C_2=4(4c_b\kappa^2\log\frac{2}{\delta})^{2r-1}(8p\kappa^2\beta_\infty)^{2-2r}$, \\ $C_3:= 4\kappa^4C_{\kappa,r}^2\log^2\frac{2}{\delta}$ ,  $C_{\delta,\kappa}= 8\kappa^4\|\mathcal{L}_\infty\|^{-1}\log^2 \frac{2}{\delta}$ , $\beta_\infty=\log \frac{4 \kappa^2(\mathcal{N}_{\mathcal{L}_\infty}(\lambda)+1)}{\delta\|\mathcal{L}_\infty\|}$, $C_{\kappa,r}$ from Proposition \ref{ineq1}, that the bias term can be bounded with probability at least $1-3\delta$ by 
\begin{align*}
\|g_\rho-\mathcal{S}_Mf_\lambda^*\|_{L^2(\rho_x)} \leq 3 R c_{r\vee 1}\lambda^{r}.
\end{align*}

\end{proposition}

\begin{proof}
We use from Assumption \ref{ass:source} that  $g_\rho=\mathcal{L}_\infty^r h$  with $\|h\|_{L^2(\rho_x)}\leq R$  to obtain, 
\begin{align}
\|g_\rho-\mathcal{S}_Mf_\lambda^*\|_{L^2(\rho_x)} &= \left\|\left(\mathcal{L}_M\phi_\lambda(\mathcal{L}_M)-I\right)\mathcal{L}_\infty^r h\right\|_{L^2(\rho_x)} \leq R \left\|r_\lambda(\mathcal{L}_M)\mathcal{L}_\infty^r\right\|\, ,
\end{align}

where $r_\lambda$ denotes the residual polynomial from \eqref{residual}. For the last term we have 
\begin{align*}
R\left\| r_\lambda(\mathcal{L}_M) \mathcal{L}_\infty^r\right\|&\leq R \left\| r_\lambda(\mathcal{L}_M) \mathcal{L}_{M,\lambda}^{(r\vee 1)}\right\|\left\|\mathcal{L}_{M,\lambda}^{-(r\vee1)}\mathcal{L}_{\infty,\lambda}^{r}\right\|\\
&\leq 3 R c_{r\vee 1}\lambda^{r},
\end{align*}
where we used for the last inequality that from \eqref{c_r} we have $\left\| r_\lambda(\mathcal{L}_M) \mathcal{L}_{M,\lambda}^{(r\vee 1)}\right\|\leq c_{r\vee 1} \lambda^{(r\vee 1)}$ and from the conditions on $M$ we have from Proposition \ref{OPbound7}  $\left\|\mathcal{L}_{M,\lambda}^{-(r\vee1)}\mathcal{L}_{\infty,\lambda}^{r}\right\| \leq 3\lambda^{-(1-r)^+}$ with probability at least $1-3\delta$.

\end{proof}
Now we want to bound the variance term. To do so we first need the following technical proposition.

\begin{proposition}
\label{T2I} 
Given the Assumptions \ref{ass:input} ,\ref{ass:kernel} , \ref{ass:source}, \ref{ass:dim}, then we have for any $s\in[0,0.5]$, $\delta \in(0,\frac{1}{11})$ with probability at least $1-11\delta$,
\begin{align}\label{a)techniqineq}
&a) \,\,\,\left\|\Sigma_M^{\frac{1}{2}-s} \phi_\lambda(\widehat\Sigma_M) \widehat{\mathcal{S}}_{M}^{*}\left(y-\widehat{\mathcal{S}}_M f_\lambda^*\right)\right\|_{\mathcal{H}_M}\leq 12D\left(\log\frac{2}{\delta}+ R c_{r\vee 1}\right)\lambda^{r-s},\\[5pt]
&b)\,\,\,\, \left\|\Sigma_M^{\frac{1}{2}-s}r_\lambda(\widehat\Sigma_M)f_\lambda^*\right\|_{\mathcal{H}_M}\leq 12DR c_{\frac{1}{2}+r   }    \lambda^{r-s},\label{b)techniqineq}
\end{align}
as long as $\nu\geq r\vee1$ and
\vspace{-0.2cm}
\begin{align*}
M&\geq 
\begin{cases}
\frac{8 p\kappa^2 \beta_\infty}{\lambda}\vee C_{\delta,\kappa} & r\in\left(0,\frac{1}{2}\right)\\
\frac{(8 p\kappa^2 \beta_\infty)\vee C_1^{\frac{1}{r}}}{\lambda}\vee \frac{C_2}{\lambda^{1+b(2r-1)} } \vee C_{\delta,\kappa} & r\in\left[\frac{1}{2},1\right] \\
\frac{C_3}{\lambda^{2r}}\vee C_{\delta,\kappa} & r \in(1,\infty),\\
\end{cases}\\
 n&\geq 
\begin{cases}
\eta_1\vee\eta_2\vee\eta_3\vee\eta_4 & r\in(0,\frac{1}{2}],\\
\eta_1\vee\eta_2\vee\eta_3\vee\eta_4 \vee\eta_5 \vee\eta_6 & r>\frac{1}{2},
\end{cases}
\end{align*}

where  $C_1=2(4\kappa\log\frac{2}{\delta})^{2r-1}(8p\kappa^2\beta_\infty)^{1-r}$ ,  $C_2=4(4c_b\kappa^2\log\frac{2}{\delta})^{2r-1}(8p\kappa^2\beta_\infty)^{2-2r}$, \\ $C_3:= 4\kappa^4C_{\kappa,r}^2\log^2\frac{2}{\delta}$ ,  $C_{\delta,\kappa}= 8\kappa^4\|\mathcal{L}_\infty\|^{-1}\log^2 \frac{2}{\delta}$ , $\beta_\infty=\log \frac{4 \kappa^2(\mathcal{N}_{\mathcal{L}_\infty}(\lambda)+1)}{\delta\|\mathcal{L}_\infty\|}$, $C_{\kappa,r}$ from Proposition \ref{ineq1}, 
\begin{center}
$\begin{aligned}
&\eta_1:= \frac{8\kappa^2 \tilde\beta}{\lambda}, 
&& \eta_2:=\frac{8QZ\kappa}{\lambda^{r+\frac{1}{2}}}, \\
&\eta_3:=\frac{128Q^2\left(1+2\log\frac{2}{\delta}\right)\mathcal{N}_{\mathcal{L}_{\infty}}(\lambda)}{\lambda^{2r}}, &&\eta_4:=\frac{72 R^2 c_{r\vee 1}^2\left(Q^2+C_{\kappa,R,D}^2\right)}{\lambda^{2r+(1-2r)^+}},\\
&\eta_5 = \frac{100\kappa^2 \mathcal{N}_{\mathcal{L}_{\infty}}(\lambda)\log^3 \frac{2}{\delta}}{\lambda} , 
&&\eta_6= \frac{8C_{\kappa,r}^2 \kappa^4 \log^2\frac{2}{\delta}}{\lambda^{2r}}
\end{aligned}$ \\\vspace{0.2cm}
\end{center}
and  $\tilde{\beta}:= \log \frac{4 \kappa^2(\left(1+2\log\frac{2}{\delta}\right)4\mathcal{N}_{\mathcal{L}_{\infty}}(\lambda)+1)}{\delta\|\mathcal{L}_\infty\|}$,  $C_{\kappa,R,D}=2 \kappa^{2r+1} R D$.

\end{proposition}

\begin{proof}
\textbf{$a)$} We start with the following decomposition

\begin{align}
\nonumber \Biggl\|\Sigma_M^{\frac{1}{2}-s}& \phi_\lambda(\widehat\Sigma_M) \widehat{\mathcal{S}}_{M}^{*}\left(y-\widehat{\mathcal{S}}_M f_\lambda^*\right)\Biggr\|_{L^2(\rho_x)} \\
\nonumber \leq& \Biggl\|\Sigma_M^{\frac{1}{2}-s}\phi_\lambda(\widehat\Sigma_M) \Sigma_{M,\lambda}^{\frac{1}{2}}\Biggr\|\left\|\Sigma_{M,\lambda}^{-\frac{1}{2}}\widehat{\mathcal{S}}_{M}^{*}\left(y-\widehat{\mathcal{S}}_M f_\lambda^*\right)\right\|_{\mathcal{H}_M} \\
=& I\cdot II. \label{I.II}
\end{align}

\begin{itemize}

\item[$I)$] From Proposition \ref{OPbound6} we have with probability at least $1-4\delta$, \\$\left\|\widehat{\Sigma}_{M,\lambda}^{-\frac{1}{2}}\Sigma_{M,\lambda}^{\frac{1}{2}}\right\| \leq 2$ . Using this inequality we obtain
\begin{align*}
\Biggl\|\Sigma_M^{\frac{1}{2}-s} \phi_\lambda(\widehat\Sigma_M) \Sigma_{M,\lambda}^{\frac{1}{2}}\Biggr\|&\leq \lambda^{-s}\Biggl\|\Sigma_{M,\lambda}^{\frac{1}{2}} \phi_\lambda(\widehat\Sigma_M) \Sigma_{M,\lambda}^{\frac{1}{2}}\Biggr\|\\
&\leq \lambda^{-s}\Biggl\|\Sigma_{M,\lambda}^{\frac{1}{2}} \phi_\lambda(\widehat\Sigma_M) \Sigma_{M,\lambda}^{\frac{1}{2}}\Biggr\|\\
&\leq \lambda^{-s}\Biggl\|\widehat\Sigma_{M,\lambda} \phi_\lambda(\widehat\Sigma_M) \Biggr\|\left\|\Sigma_{M,\lambda}^{\frac{1}{2}} \widehat{\Sigma}_{M,\lambda}^{-\frac{1}{2}}  \right\|^2\leq \lambda^{-s}4 D,
\end{align*}
where $D$ is defined in \eqref{def.phi}.
\item[$II)$] For the second term we have 

\begin{align*}
&\left\|\Sigma_{M,\lambda}^{-\frac{1}{2}}\widehat{\mathcal{S}}_{M}^{*}\left(y-\widehat{\mathcal{S}}_M f_\lambda^*\right)\right\|_{\mathcal{H}_M}\\
&\leq \left\|\Sigma_{M,\lambda}^{-\frac{1}{2}}\widehat{\mathcal{S}}_{M}^{*}\left(y-\bar{g}_\rho\right)\right\|_{\mathcal{H}_M}+\left\|\Sigma_{M,\lambda}^{-\frac{1}{2}}\widehat{\mathcal{S}}_{M}^{*}\left(\bar{g}_\rho-\widehat{\mathcal{S}}_M f_\lambda^*\right)\right\|_{\mathcal{H}_M}\\
&:= i+ii
\end{align*}
For the first norm $i)$ we use the bound of  Proposition \ref{concentrationineq1} together with the bound of \ref{prop:effecdim2}:
$$
\mathcal{N}_{\mathcal{L}_{M}}(\lambda)\leq  \left(1+2\log\frac{2}{\delta}\right)4\mathcal{N}_{\mathcal{L}_{\infty}}(\lambda),
$$
to obtain with probability at least $1-3\delta$,
\begin{align*}
\left\|\Sigma_{M,\lambda}^{-\frac{1}{2}}\widehat{\mathcal{S}}_{M}^{*}\left(y-\bar{g}_\rho\right)\right\|_{\mathcal{H}_M}&\leq \left(\frac{4QZ\kappa}{\sqrt{\lambda}n}+\frac{4Q\sqrt{\mathcal{N}_{\mathcal{L}_M}(\lambda)}}{\sqrt{n}}\right) \log \frac{2}{\delta}\\
&\leq\left(\frac{4QZ\kappa}{\sqrt{\lambda}n}+\frac{8Q\sqrt{\left(1+2\log\frac{2}{\delta}\right)\mathcal{N}_{\mathcal{L}_{\infty}}(\lambda)}}{\sqrt{n}}\right) \log \frac{2}{\delta}\\
&\leq \lambda^{r}\log \frac{2}{\delta},
\end{align*}
where we used in the last inequality that $n\geq\eta_2\vee\eta_3 :=\frac{8QZ\kappa}{\lambda^{r+\frac{1}{2}}}\vee\frac{128Q^2\left(1+2\log\frac{2}{\delta}\right)\mathcal{N}_{\mathcal{L}_{\infty}}(\lambda)}{\lambda^{2r}}$.

For the second norm $ii)$ we first use that
\begin{align*}
\left\|\Sigma_{M,\lambda}^{-\frac{1}{2}}\widehat{\mathcal{S}}_{M}^{*}\right\|^2&=
\left\|(\widehat{\mathcal{S}}_M^*\widehat{\mathcal{S}}_M+\lambda)^{-1/2}\widehat{\mathcal{S}}_M^*\right\|^2\\
&=\left\|(\widehat{\mathcal{S}}_M^*\widehat{\mathcal{S}}_M+\lambda)^{-1/2}\widehat{\mathcal{S}}_M^*\widehat{\mathcal{S}}_M(\widehat{\mathcal{S}}_M^*\widehat{\mathcal{S}}_M+\lambda)^{-1/2}\right\|\\
&=\left\|\widehat{\mathcal{S}}_M^*\widehat{\mathcal{S}}_M(\widehat{\mathcal{S}}_M^*\widehat{\mathcal{S}}_M+\lambda)^{-1}\right\|\leq1,
\end{align*}
to obtain together with the bound of Proposition \ref{concentrationineq2}, that with probability at least $1-\delta$,
\begin{align*}
&\left\|\Sigma_{M,\lambda}^{-\frac{1}{2}}\widehat{\mathcal{S}}_{M}^{*}\left(\bar{g}_\rho-\widehat{\mathcal{S}}_M f_\lambda^*\right)\right\|_{\mathcal{H}_M}\\
&\leq \frac{1}{\sqrt{n}}\left\|\bar{g}_\rho-\widehat{\mathcal{S}}_M f_\lambda^*\right\|_2\\
&\leq \sqrt{\left|\frac{1}{n}\left\|\bar{g}_\rho-\widehat{\mathcal{S}}_M f_\lambda^*\right\|_2^2-\left\|g_\rho-\mathcal{S}_M f_\lambda^*\right\|_{L^2(\rho_x)}^2\right|}+\left\|g_\rho-\mathcal{S}_M f_\lambda^*\right\|_{L^2(\rho_x)}\\
&\leq \sqrt{2\left(\frac{4\left(Q^2+ C_{\kappa,R,D}^2 \,\lambda^{-2(\frac{1}{2}-r)^+}\right)}{n}+\frac{\sqrt{2}\left(Q+C_{\kappa,R,D}  \,\lambda^{-(\frac{1}{2}-r)^+}\right)\left\|g_\rho-\mathcal{S}_M f_\lambda^*\right\|_{L^2(\rho_x)}}{\sqrt{n}}\right) \log \frac{2}{\delta}}\,\,+\\[5pt]
&\,\,\,\,\,\,\,\,\,\,\,\left\|g_\rho-\mathcal{S}_M f_\lambda^*\right\|_{L^2(\rho_x)}.
\end{align*}

From Proposition \ref{mainprop2} we further obtain with probability at least $1-3\delta$,
\begin{align*}
&\left\|\Sigma_{M,\lambda}^{-\frac{1}{2}}\widehat{\mathcal{S}}_{M}^{*}\left(\bar{g}_\rho-\widehat{\mathcal{S}}_M f_\lambda^*\right)\right\|_{\mathcal{H}_M}\\
&\leq \sqrt{2\left(\frac{4\left(Q^2+ C_{\kappa,R,D}^2 \,\lambda^{-(1-2r)^+}\right)}{n}+\frac{\sqrt{2}\left(Q+C_{\kappa,R,D}  \,\lambda^{-(\frac{1}{2}-r)^+}\right)3 R c_{r\vee 1}\lambda^{r}}{\sqrt{n}}\right) \log \frac{2}{\delta}}+3 R c_{r\vee 1}\lambda^{r}\\
&\leq \lambda^{r} \left(\sqrt{\log\frac{2}{\delta}}+3 R c_{r\vee 1}\right),
\end{align*}
where we used in the last inequality that $n\geq \eta_4:=\frac{72 R^2 c_{r\vee 1}^2\left(Q^2+C_{\kappa,R,D}^2\right)}{\lambda^{2r+(1-2r)^+}}$.
Therefore we have for the second term
$$
II \leq \lambda^{r} \left(\log \frac{2}{\delta}+\sqrt{\log\frac{2}{\delta}}+3 R c_{r\vee 1}\right)
$$
\end{itemize}
Plugging the bounds of $I$ and $II$ in \eqref{I.II} proves the inequality from the statement. Collecting all the concentration inequalities in this proof and applying Proposition \ref{conditioning}, we obtain that the inequality \eqref{a)techniqineq} holds true with probability at least $1-11\delta$.\\

\textbf{$b)$} Using Mercers theorem (see for example \cite{steinwart2008support}) we have with probability at least $1-4\delta$,
\begin{align*}
\left\|\Sigma_M^{\frac{1}{2}-s} r_\lambda(\widehat\Sigma_M)f_\lambda^*\right\|_{\mathcal{H}_M}&\leq\lambda^{-s}\left\| \Sigma_{M,\lambda}^{\frac{1}{2}} r_\lambda(\widehat\Sigma_M)f_\lambda^*\right\|_{\mathcal{H}_M}\\
&\leq \lambda^{-s}\left\|\widehat{\Sigma}_{M,\lambda}^{-\frac{1}{2}}\Sigma_{M,\lambda}^{\frac{1}{2}}\right\| \left\| \widehat\Sigma_{M,\lambda}^{\frac{1}{2}} r_\lambda(\widehat\Sigma_M)f_\lambda^*\right\|_{\mathcal{H}_M}\\
&\leq 2 \lambda^{-s} \left\| \widehat\Sigma_{M,\lambda}^{\frac{1}{2}} r_\lambda(\widehat\Sigma_M)f_\lambda^*\right\|_{\mathcal{H}_M},
\end{align*}
where we used Proposition \ref{OPbound6} for the last inequality. 
To continue we write out the definition of $f_\lambda^*$ to obtain
\begin{align}
\nonumber&2 \lambda^{-s}\left\| \widehat\Sigma_{M,\lambda}^{\frac{1}{2}} r_\lambda(\widehat\Sigma_M)f_\lambda^*\right\|_{\mathcal{H}_M}\\
&\leq 2 R \lambda^{-s} \left\|\widehat\Sigma_{M,\lambda}^{\frac{1}{2}} r_\lambda(\widehat\Sigma_M) \mathcal{S}_M^* \phi(\mathcal{L}_M) \mathcal{L}_\infty^r\right\|.\label{casesT2}
\end{align}
To bound the last term we need to differ between the following two cases. 
\begin{itemize}

\item CASE ($r\leq \frac{1}{2}$) : To bound the norm of \eqref{casesT2} for $r\leq\frac{1}{2}$ we start with 
\begin{align*}
&\left\|\widehat\Sigma_{M,\lambda}^{\frac{1}{2}} r_\lambda(\widehat\Sigma_M) \mathcal{S}_M^* \phi_\lambda(\mathcal{L}_M) \mathcal{L}_\infty^r\right\|\\
&\leq\left\|\widehat\Sigma_{M,\lambda}^{\frac{1}{2}} r_\lambda(\widehat\Sigma_M) \mathcal{S}_M^* \phi_\lambda(\mathcal{L}_M) \mathcal{L}_{M,\lambda}^{r}\right\|\left\|\mathcal{L}_{M,\lambda}^{-r}\mathcal{L}_{\infty,\lambda}^{r}\right\|\\
&= \left\|\widehat\Sigma_{M,\lambda}^{\frac{1}{2}} r_\lambda(\widehat\Sigma_M) \Sigma_{M,\lambda}^{r}\mathcal{S}_M^* \phi_\lambda(\mathcal{L}_M) \right\|\left\|\mathcal{L}_{M,\lambda}^{-r}\mathcal{L}_{\infty,\lambda}^{r}\right\|.\\
&\leq  \left\|\widehat\Sigma_{M,\lambda}^{\frac{1}{2}} r_\lambda(\widehat\Sigma_M) \Sigma_{M,\lambda}^{r}\right\| \left\| \mathcal{L}_M^{\frac{1}{2}}\phi_\lambda(\mathcal{L}_M) \right\|\left\|\mathcal{L}_{M,\lambda}^{-r}\mathcal{L}_{\infty,\lambda}^{r}\right\|.
\end{align*}

From Proposition \ref{ineqvolkan} we have $ \left\| \mathcal{L}_M^{\frac{1}{2}}\phi_\lambda(\mathcal{L}_M) \right\| \leq D \lambda^{-\frac{1}{2}}$ and from Proposition \ref{ineq2} together with \ref{OPbound5} we have $\left\|\mathcal{L}_{M,\lambda}^{-r}\mathcal{L}_{\infty,\lambda}^{r}\right\|\leq\left\|\mathcal{L}_{M,\lambda}^{-\frac{1}{2}}\mathcal{L}_{\infty,\lambda}^{\frac{1}{2}}\right\|^{2r}\leq 2^{2r}\leq 2$ with probability at least $1-4\delta$.  Using those bounds we obtain for \eqref{casesT2}

\begin{align}
\left\|\Sigma_M^{\frac{1}{2}-s}r_\lambda(\widehat\Sigma_M)f_\lambda^*\right\|_{\mathcal{H}_M}\leq 4DR \lambda^{-s-\frac{1}{2}}\left\|\widehat\Sigma_{M,\lambda}^{\frac{1}{2}} r_\lambda(\widehat\Sigma_M) \Sigma_{M,\lambda}^{r}\right\|. \label{TIIcase1}
\end{align}

It remains to bound $\left\|\widehat\Sigma_{M,\lambda}^{\frac{1}{2}} r_\lambda(\widehat\Sigma_M) \Sigma_{M,\lambda}^{r}\right\|$ . Using again Proposition \ref{OPbound6} we have that $\left\|\widehat{\Sigma}_{M,\lambda}^{-r} \Sigma_{M,\lambda}^{r}\right\|\leq \left\|\widehat{\Sigma}_{M,\lambda}^{-\frac{1}{2}} \Sigma_{M,\lambda}^{\frac{1}{2}}\right\|^{2r}\leq 2$. From this bound together with \eqref{c_r}
we obtain 
\begin{align*}
\left\|\widehat\Sigma_{M,\lambda}^{\frac{1}{2}} r_\lambda(\widehat\Sigma_M) \Sigma_{M,\lambda}^{r}\right\|&\leq \left\|\widehat\Sigma_{M,\lambda}^{\frac{1}{2}} r_\lambda(\widehat\Sigma_M) \widehat{\Sigma}_{M,\lambda}^{r}\right\|\left\|\widehat{\Sigma}_{M,\lambda}^{-r} \Sigma_{M,\lambda}^{r}\right\| \\
&\leq c_{\frac{1}{2}+r} \lambda^{\frac{1}{2}+r} \left\|\widehat{\Sigma}_{M,\lambda}^{-r} \Sigma_{M,\lambda}^{r}\right\|\\
&\leq 2c_{\frac{1}{2}+r} \lambda^{\frac{1}{2}+r} 
\end{align*}

Plugging the above bound into \eqref{TIIcase1} gives
\begin{align*}
\left\|\Sigma_M^{\frac{1}{2}-s} r_\lambda(\widehat\Sigma_M)f_\lambda^*\right\|_{\mathcal{H}_M}\leq 8DR c_{\frac{1}{2}+r   }     \lambda^{r-s}.
\end{align*}

\item CASE ($r>\frac{1}{2}$) :  To bound the norm of \eqref{casesT2} for $r>\frac{1}{2}$ we start similar with 
\begin{align*}
&\left\|\widehat\Sigma_{M,\lambda}^{\frac{1}{2}} r_\lambda(\widehat\Sigma_M) \mathcal{S}_M^* \phi_\lambda(\mathcal{L}_M) \mathcal{L}_\infty^r\right\|\\
&\leq\left\|\widehat\Sigma_{M,\lambda}^{\frac{1}{2}} r_\lambda(\widehat\Sigma_M) \mathcal{S}_M^* \phi_\lambda(\mathcal{L}_M) \mathcal{L}_{M,\lambda}^{(r\vee1)}\right\|\left\|\mathcal{L}_{M,\lambda}^{-(r\vee1)}\mathcal{L}_{\infty,\lambda}^{r}\right\|\\
&= \left\|\widehat\Sigma_{M,\lambda}^{\frac{1}{2}} r_\lambda(\widehat\Sigma_M) \Sigma_{M,\lambda}^{(r\vee1)}\mathcal{S}_M^* \phi_\lambda(\mathcal{L}_M) \right\|\left\|\mathcal{L}_{M,\lambda}^{-(r\vee1)}\mathcal{L}_{\infty,\lambda}^{r}\right\|.\\
&\leq  \left\|\widehat\Sigma_{M,\lambda}^{\frac{1}{2}} r_\lambda(\widehat\Sigma_M) \Sigma_{M,\lambda}^{(r\vee1)-\frac{1}{2}}\right\| \left\| \mathcal{L}_M\phi_\lambda(\mathcal{L}_M) \right\|\left\|\mathcal{L}_{M,\lambda}^{-(r\vee1)}\mathcal{L}_{\infty,\lambda}^{r}\right\|.
\end{align*}

By definition of a spectral method $\phi_t$ and  \ref{OPbound7} we have $ \left\| \mathcal{L}_M\phi_\lambda(\mathcal{L}_M) \right\| \leq D $ and $\left\|\mathcal{L}_{M,\lambda}^{-(r\vee1)}\mathcal{L}_{\infty,\lambda}^{r}\right\|\leq3 \lambda^{- (1-r)^+}$. Using those bounds we obtain for \eqref{casesT2}

\begin{align}
\left\|\Sigma_M^{\frac{1}{2}-s}r_\lambda(\widehat\Sigma_M)f_\lambda^*\right\|_{\mathcal{H}_M}\leq \frac{6DR}{\lambda^{s+(1-r)^+}} \left\|\widehat\Sigma_{M,\lambda}^{\frac{1}{2}} r_\lambda(\widehat\Sigma_M) \Sigma_{M,\lambda}^{(r\vee1)-\frac{1}{2}}\right\|. \label{TIIcase2}
\end{align}

It remains to bound $\left\|\widehat\Sigma_{M,\lambda}^{\frac{1}{2}} r_\lambda(\widehat\Sigma_M) \Sigma_{M,\lambda}^{q}\right\|$  with $q:= (r\vee1)-\frac{1}{2}$. From \eqref{c_r} and Proposition \ref{OPbound8} we have with probability at least $1-\delta$,
\begin{align*}
\left\|\widehat\Sigma_{M,\lambda}^{\frac{1}{2}} r_\lambda(\widehat\Sigma_M) \Sigma_{M,\lambda}^{q}\right\|&\leq \left\|\widehat\Sigma_{M,\lambda}^{\frac{1}{2}} r_\lambda(\widehat\Sigma_M) \widehat{\Sigma}_{M,\lambda}^{q}\right\|\left\|\widehat{\Sigma}_{M,\lambda}^{-q} \Sigma_{M,\lambda}^{q}\right\| \\
&\leq c_{r \vee 1} \lambda^{(r\vee1)} \left\|\widehat{\Sigma}_{M,\lambda}^{-q} \Sigma_{M,\lambda}^{q}\right\|\\
&\leq2 c_{r\vee 1} \lambda^{(r\vee1)} 
\end{align*}

Plugging the above bound into \eqref{TIIcase2} gives
\begin{align*}
\left\|\mathcal{S}_M r_\lambda(\widehat\Sigma_M)f_\lambda^*\right\|_{L^2(\rho_x)}\leq 12DR c_{\frac{1}{2}+r   }    \lambda^{r-s}.
\end{align*}

\end{itemize}
Combining the bounds of both cases proves the claim. If we collect all the concentration inequalities used, and applying Proposition \ref{conditioning}, we obtain that the statement \eqref{b)techniqineq} holds true with probability at least $1-8\delta$.
\end{proof}

Now we are able to bound the variance term.

\begin{proposition}
\label{mainprop}
Provided the same assumptions of Proposition \ref{T2I}, we have for any $s\in[0,0.5]$, $\delta \in(0,\frac{1}{19})$ with probability at least $1-19\delta$, 
\begin{align*}
\left\|\Sigma_M^{\frac{1}{2}-s} (f_\lambda^M-f_\lambda^*)\right\|_{\mathcal{H}_M}\leq \left( 12D\left(\log\frac{2}{\delta}+ R c_{r\vee 1}\right)+ 12DR c_{\frac{1}{2}+r }  \right)\lambda^{r-s}.
\end{align*}
\end{proposition}

\begin{proof}
We start with the following decomposition
\begin{align}
&\left\|\Sigma_M^{\frac{1}{2}-s} (f_\lambda^M-f_\lambda^*)\right\|_{\mathcal{H}_M} \\
&\leq\left\|\Sigma_M^{\frac{1}{2}-s}\left(\phi_\lambda(\widehat\Sigma_M) \widehat{\cS}_{M}^{*} y-\phi_\lambda(\widehat\Sigma_M) \widehat{\Sigma}_{M}f_\lambda^*\right)\right\|_{\mathcal{H}_M}+\left\|\Sigma_M^{\frac{1}{2}-s} \left(\phi_\lambda(\widehat\Sigma_M) \widehat{\Sigma}_{M}-I\right)f_\lambda^*\right\|_{\mathcal{H}_M}\\
&=\left\|\Sigma_M^{\frac{1}{2}-s}\phi_\lambda(\widehat\Sigma_M) \widehat{\cS}_{M}^{*}\left(y-\widehat{\mathcal{S}}_M f_\lambda^*\right)\right\|_{\mathcal{H}_M}+\left\|\Sigma_M^{\frac{1}{2}-s} r_\lambda(\widehat\Sigma_M)f_\lambda^*\right\|_{\mathcal{H}_M}\\
&:=I+II.
\end{align}

For $I)$ we have from Proposition \ref{T2I}  a) with probability at least $1-11\delta$,
$$I\leq 12D\left(\log\frac{2}{\delta}+ R c_{r\vee 1}\right)\lambda^{r-s}$$
and for $II)$ we have from Proposition \ref{T2I} b) with probability at least $1-8\delta$,
$$II\leq 12DR c_{\frac{1}{2}+r }  \lambda^{r-s}.$$
Combining those bounds proves the claim.
\end{proof}

\begin{theorem}
\label{theo2}
Provided all the assumptions of Proposition \ref{T2I}  we have with probability at least $1-22\delta$,\\
\begin{align*}
\|g_\rho-\mathcal{S}_M f_\lambda^M\|_{L^2(\rho_x)}\leq \left(3 R c_{r\vee 1}+ 12D\left(\log\frac{2}{\delta}+ R c_{r\vee 1}\right)+ 12DR c_{\frac{1}{2}+r }  \right)\lambda^{r}.
\end{align*}

\end{theorem}

\begin{proof}
We start with the following decomposition
\begin{align}
\|g_\rho-\mathcal{S}_M f_\lambda^M\|_{L^2(\rho_x)}\leq\|g_\rho-\mathcal{S}_Mf_\lambda^*\|_{L^2(\rho_x)}+\|\mathcal{S}_M f_\lambda^M-\mathcal{S}_Mf_\lambda^*\|_{L^2(\rho_x)}
:=T_1+T_2 .\label{Maindecomposition}
\end{align}

We will now bound $T_1$ and $T_2$ separately :
\begin{itemize}

\item[$T_1)$] For the first term of \eqref{Maindecomposition} we have from Proposition \ref{mainprop2} with probability at least $1-3\delta$,
\begin{align}
\|g_\rho-\mathcal{S}_Mf_\lambda^*\|_{L^2(\rho_x)} \leq 3 R c_{r\vee 1}\lambda^{r}\, .\label{(T_1)}
\end{align}

\item[$T_2)$] For the second norm in \eqref{Maindecomposition} we have from Mercers Theorem (see for example \cite{steinwart2008support}) and Proposition \ref{mainprop} with probability at least $1-19\delta$,
\begin{align*}
\|\mathcal{S}_M f_\lambda^M-\mathcal{S}_Mf_\lambda^*\|_{L^2(\rho_x)} = \left\|\Sigma_M^{\frac{1}{2}} (f_\lambda^M-f_\lambda^*)\right\|_{\mathcal{H}_M} \leq \left( 12D\left(\log\frac{2}{\delta}+ R c_{r\vee 1}\right)+ 12DR c_{\frac{1}{2}+r }  \right)\lambda^{r}.
\end{align*}
Combining this bound with the bound of $T_1$ proves the claim.

\end{itemize}

\end{proof}

\begin{proof}[Proof of Theorem \ref{theo1}]
The proof follows from \ref{theo2} . First we need to check if $\lambda = C n^{-\frac{1}{2r+b}}\log^3\frac{2}{\delta}$ for some $C>0$ fulfills the conditions of \ref{T2I} on $n$ and $M$. Using the bound of \ref{effecDim}  we have that the condition  $n\geq \eta_1\vee\eta_2\vee\eta_3\vee\eta_4 \vee\eta_6$ is fulfilled if
$$
n\geq c_1 \left(\frac{\log(\lambda^{-1})}{\lambda} +\frac{1}{\lambda^{2r+b}}\right)\log^3\frac{2}{\delta},
$$
where 
$$
c_1=\max\left\{8\kappa^2,\,\, 8QZ\kappa,\,\, 382Q^2c_b,\,\, 72 R^2 c_{r\vee 1}^2\left(Q^2+C_{\kappa,R,D}^2\right)\,\, \right\} \cdot \max\left\{1,\,\log \frac{48 \kappa^2c_b}{\|\mathcal{L}_\infty\|}\right\}
$$

Therefore for the case $r\leq \frac{1}{2}$ it is enough to assume $\lambda= 2c_1 n^{-\frac{1}{2r+b}} \log^3\frac{2}{\delta}$ as long as $n\geq n_0$
where $n_0\geq n_0^{\frac{1}{2r+b}} \log(n_0)$ or equivalent $n_0\geq e^{\frac{2r+b}{2r+b-1}}$.  In case $r>\frac{1}{2}$ it remains to check if 
$n\geq \eta_5$. This holds if $n\geq  \frac{c_2}{\lambda^{1+b}} \log^3\frac{2}{\delta}$ with $c_2:=100\kappa^2c_b$ and therefore the condition on $n$ is fulfilled if 
$\lambda = C n^{-\frac{1}{2r+b}}\log^3\frac{2}{\delta}$  where $C:= \max\{2c_1,c_2\}$. The condition on $M$ is fulfilled if

\begin{align*}
M&\geq 
\begin{cases}
\frac{8 p\kappa^2 \beta_\infty}{\lambda}\vee C_{\delta,\kappa} & r\in\left(0,\frac{1}{2}\right)\\
\frac{8 p\kappa^2 \beta_\infty\vee C_1^{\frac{1}{r}}\vee C_2}{\lambda^{1+b(2r-1)} } \vee C_{\delta,\kappa} & r\in\left[\frac{1}{2},1\right] \\[5pt]
\frac{8 p\kappa^2 \beta_\infty\vee C_3}{\lambda^{2r}}\vee C_{\delta,\kappa} & r \in(1,\infty).\\
\end{cases}\\
 \end{align*}
Using $\lambda = C n^{-\frac{1}{2r+b}}\log^3\frac{2}{\delta}$ we have that the condition is fulfilled if

\begin{align*}
M\geq \frac{8 p\kappa^2 \beta_\infty\vee C_{\delta,\kappa} \vee \left(C_1^{\frac{1}{r}}\vee C_2\right)\mathbbm{1}_{r<1} \vee C_3}{C\log^3\frac{2}{\delta}}\cdot
\begin{cases}
n^{\frac{1}{2r+b}}& r\in\left(0,\frac{1}{2}\right)\\
n^{\frac{1+b(2r-1)}{2r+b}}  & r\in\left[\frac{1}{2},1\right] \\
n^{\frac{2r}{2r+b}} & r \in(1,\infty)\,.\\
\end{cases}\\
\end{align*}

Note that 
\begin{align*}
\frac{8 p\kappa^2 \beta_\infty\vee C_{\delta,\kappa} \vee \left(C_1^{\frac{1}{r}}\vee C_2\right)\mathbbm{1}_{\{r<1\}} \vee C_3}{C\log^3\frac{2}{\delta}} \leq \tilde{C} \log(n),
\end{align*}
for 
$\tilde{C}= 8p\kappa^2\cdot \max\left\{1,\,\log \frac{48 \kappa^2c_b}{\|\mathcal{L}_\infty\|}\right\}\cdot\frac{ 8\kappa^4\|\mathcal{L}_\infty\|^{-1}  \vee  8\kappa  \vee   16c_b\kappa^2\vee 4\kappa^4C_{\kappa,r}^2
}{C}$

Therefore the condition on $M$ holds true if 
\begin{align*}
M\geq \tilde{C}\log(n) \cdot \begin{cases}
n^{\frac{1}{2r+b}}& r\in\left(0,\frac{1}{2}\right)\\
n^{\frac{1+b(2r-1)}{2r+b}}  & r\in\left[\frac{1}{2},1\right] \\
n^{\frac{2r}{2r+b}} & r \in(1,\infty)\,.\\
\end{cases}\\
\end{align*}

Theorem \ref{theo2} now states with probability at least $1-22\delta$,
\begin{align}
\|g_\rho-\mathcal{S}_M f_\lambda^M\|_{L^2(\rho_x)}&\leq \left(3 R c_{r\vee 1}+ 12D\left(\log\frac{2}{\delta}+ R c_{r\vee 1}\right)+ 12DR c_{\frac{1}{2}+r }  \right)\lambda^{r} \label{statement}\\
&= \left(3 R c_{r\vee 1}+ 12D\left(\log\frac{2}{\delta}+ R c_{r\vee 1}\right)+ 12DR c_{\frac{1}{2}+r }  \right)\left(C n^{-\frac{1}{2r+b}}\log^3\frac{2}{\delta}\right)^{r}\\
&\leq \bar{C} \log^{3r+1} \left(\frac{2}{\delta}\right) \, n^{-\frac{r}{2r+b}},
\end{align}
where 
$$
\bar{C}:=\left(3 R c_{r\vee 1}+ 12D\left(1+ R c_{r\vee 1}\right)+ 12DR c_{\frac{1}{2}+r }  \right)C^r. 
$$ 
Redefining  $\tilde{\delta}=22\delta\in(0,1)$ proves the statement.

\end{proof}

%% file: T.U.tex
%%%%%%%%%%%%%%%%%%%%%   TECHNISCHE UNGLEICHUNGEN   %%%%%%%%%%%%%%%%
%%%%%%%%%%%%%%%%%%%%%%%%%%%%%%%%%%%%%%%%%%%%%%%%%%%%%%%%%%%%%%%%%%%%%%%%%%%%%%%%%%%%%%%%%%%%%%%%%%%%%%%%%%%%%%%%%%%%%%%%%%%%%%%%%%%%%%%%%%%%%%%%%%%%%%%%%%%%%%%%%%%%%%%%%%%%%%%%%%%%%%%%%%%%%%%%%%%%%%%%%%%%%%%%%%%%%%%%%%%%%%%%%%%%%%%%%%%%%%%%%%%%%%%%%%%%%%%%%%%%%%%%%%%%%%%%%%%%%%%%%%%%%%%%%%%%%%%%%%%%%%%%%%%%%%%%%%%%%%%%%%%%%%%%%%%%%%%%%%%%%%%%%%%%%%%%%%%%%%%%%%%%%%%%%%%%%%%%%%%%%%%%%%%%%%%%%%%%%%%%%%%%%%%%%%%%%%%%%%%%%%%%%%%%%%%%%%%%%%%%%%%%%%%%%%%%%%%%

\subsection{Appendix II}

\begin{proposition}
\label{conditioning}
Let  $E_i$ be events with probability at least $1-\delta_i$ and set
$$
E:=\bigcap^k_{i=1}E_i
$$
If we can show for some event $A$ that $\mathbb{P}(A|E)\geq 1-\delta$  then we also have
\begin{align*}
\mathbb{P}(A)&\geq\int_{E}\mathbb{P}(A|\omega)d\mathbb{P}(\omega)
\geq (1-\delta)\mathbb{P}(E)\\
&=(1-\delta)\left(1-\mathbb{P}\left(\bigcup_{i=1}^k (\Omega/E_i)\right)\right)\geq(1-\delta)\left(1-\sum_{i=1}^k\delta_i\right).
\end{align*}
\end{proposition}

\begin{proposition}[ \cite{aleksandrov2009operatorholderzygmundfunctions}, \cite{Muecke2017op.rates} (Proposition B.1.) ]
\label{ineq1}
Let $B_{1}, B_{2}$ be two non-negative self-adjoint operators on some Hilbert space with $\left\|B_{j}\right\| \leq a, j=1,2$, for some non-negative a.
\begin{itemize}
\item[(i)] If $0 \leq r \leq 1$, then
$$
\left\|B_{1}^{r}-B_{2}^{r}\right\| \leq C_{r}\left\|B_{1}-B_{2}\right\|^{r},
$$
for some $C_{r}<\infty$.
\item[(ii)] If $r>1$, then
$$
\left\|B_{1}^{r}-B_{2}^{r}\right\| \leq C_{a, r}\left\|B_{1}-B_{2}\right\|,
$$
for some $C_{a, r}<\infty$. 
\end{itemize}
\end{proposition}

\begin{proposition}[Fujii et al., 1993, Cordes inequality]
\label{ineq2}
Let $A$ and $B$ be two positive bounded linear operators on a separable Hilbert space. Then
$$
\left\|A^s B^s\right\| \leq\|A B\|^s, \quad \text { when } 0 \leq s \leq 1 .
$$

\end{proposition}

\begin{proposition}[\cite{features} (Proposition 9)]
\label{ineq3}
Let $\mathcal{H}, \mathcal{K}$ be two separable Hilbert spaces and $X, A$ be bounded linear operators, with $A: \mathcal{H} \rightarrow \mathcal{K}$ and $B: \mathcal{H} \rightarrow \mathcal{H}$ be positive semidefinite.
$$
\left\|A B^\sigma\right\| \leq\|A\|^{1-\sigma}\|A B\|^\sigma, \quad \forall \sigma \in[0,1] .
$$

\end{proposition}

\begin{proposition}
\label{eq4}
Let $H_1, H_2$ be two separable Hilbert spaces and $\mathcal{S}: H_1 \rightarrow H_2$ a compact operator. Then for any function $f:[0,\|\mathcal{S}\|] \rightarrow[0, \infty[$,
$$
f\left(\mathcal{S} \mathcal{S}^*\right) \mathcal{S}=\mathcal{S} f\left(\mathcal{S}^* \mathcal{S}\right)
$$

\end{proposition}
\begin{proof}
The result can be proved using singular value decomposition of a compact operator.
\end{proof}

\begin{proposition}[\cite{spectral.rates} (Lemma 10)]
\label{ineqvolkan}
 Let $L$ be a compact, positive operator on a separable Hilbert space $H$ such that $\|L\| \leq \kappa^2$. Then for any $\lambda \geq 0$,
 \begin{align*}
\left\|(L+\lambda)^\alpha \phi_\lambda(L)\right\| &\leq 2 D \lambda^{-(1-\alpha)}, \quad \forall \alpha \in[0,1],\\
\left\|L^\alpha \phi_\lambda(L)\right\| &\leq  D \lambda^{-(1-\alpha)}, \quad \forall \alpha \in[0,1],
 \end{align*}

 where $D$ is defined in \eqref{def.phi}.
\end{proposition}

\begin{proposition}
\label{ineq5}
With probability at least $1-\delta$ we have
\begin{align*}
\|f^*_\lambda\|_\infty &\leq  2 \kappa^{2r+1} R D \,\lambda^{-(\frac{1}{2}-r)^+},\\
\|f^*_\lambda\|_{\mathcal{H}_M} &\leq 2 \kappa^{2r} R D \,\lambda^{-(\frac{1}{2}-r)^+},\\
\end{align*}
as long as $M\geq \frac{8 p\kappa^2 \beta_\infty}{\lambda}$, with $\beta_\infty=\log \frac{4 \kappa^2(\mathcal{N}_{\mathcal{L}_\infty}(\lambda)+1)}{\delta\|\mathcal{L}_\infty\|} $.

\end{proposition}

\begin{proof}
Note that $f_\lambda^*\in \mathcal{H}_M$. Therefore we obtain from the reproducing property and the definition  $g_\rho=\mathcal{L}^r_\infty h $,  for any $x\in\mathcal{X}$:
\begin{align*}
\|f_\lambda^*(x)\|_\mathcal{Y}&=\| K_{M,x}^*f_\lambda^* \|_\mathcal{Y}\\
&\leq \kappa \left\|f_\lambda^*\right\|_{\mathcal{H}_M} \\
&= \kappa \left\|\mathcal{S}^*_M\phi_\lambda(\mathcal{L}_M) g_\rho \right\|_{\mathcal{H}_M} \\
&=  \kappa \left\|\mathcal{L}_M^{\frac{1}{2}}\phi_\lambda(\mathcal{L}_M) \mathcal{L}_\infty^r h \right\|_{L^2(\rho_x)}.
\end{align*}
Using the assumption $\|h\|_{L^2(\rho_x)}\leq R$ we therefore have
\begin{align}
&\|f_\lambda^*\|_\infty
\leq  \kappa  R \left\|\mathcal{L}_M^{\frac{1}{2}}\phi_\lambda(\mathcal{L}_M) \mathcal{L}_{M,\lambda}^{(r\wedge\frac{1}{2})}\right\|  \left\|\mathcal{L}_{M,\lambda}^{-(r\wedge\frac{1}{2})}\mathcal{L}_{\infty}^{r}\right\|
= \kappa  R\,\, I\cdot II\label{fbound},\\
&\|f_\lambda^*\|_{\mathcal{H}_M}
\leq   R \left\|\mathcal{L}_M^{\frac{1}{2}}\phi_\lambda(\mathcal{L}_M) \mathcal{L}_{M,\lambda}^{(r\wedge\frac{1}{2})}\right\|  \left\|\mathcal{L}_{M,\lambda}^{-(r\wedge\frac{1}{2})}\mathcal{L}_{\infty}^{r}\right\|
=   R\,\, I\cdot II\label{fbound2}.
\end{align}
\begin{itemize}
\item [$I)$] For the first norm in \eqref{fbound} we have from Proposition \ref{ineqvolkan} that 
\begin{align*}
I&=\left\|\mathcal{L}_M^{\left(\frac{1}{2}+\left(r\wedge\frac{1}{2}\right)\right)}\phi_\lambda(\mathcal{L}_M) \right\|
\leq
\begin{cases}
D & r\geq \frac{1}{2}\\
D  \lambda^{r-\frac{1}{2}} & r<\frac{1}{2} 
\end{cases}\\
&\leq D \lambda^{-(\frac{1}{2}-r)^+} .
\end{align*}

\item [$II)$] For the second norm in \eqref{fbound} we have from Proposition \ref{ineq2} and \ref{OPbound5} that with probability at least $1-\delta$
\begin{align*}
II&= 
\begin{cases}
\left\|\mathcal{L}_{M,\lambda}^{-\frac{1}{2}}\mathcal{L}_{\infty,\lambda}^{r}\right\| \leq \left\|\mathcal{L}_{M,\lambda}^{-\frac{1}{2}}\mathcal{L}_{\infty,\lambda}^{\frac{1}{2}}\right\|\left\|\mathcal{L}_{\infty}^{r-\frac{1}{2}}\right\|\leq 2\kappa^{2r-1}& r\geq \frac{1}{2}\\[9pt]
\left\|\mathcal{L}_{M,\lambda}^{-r}\mathcal{L}_{\infty,\lambda}^{r}\right\| \leq \left\|\mathcal{L}_{M,\lambda}^{-\frac{1}{2}}\mathcal{L}_{\infty,\lambda}^{\frac{1}{2}}\right\|^{2r} \leq 4^r \leq 2 & r<\frac{1}{2} 
\end{cases}\\[5pt]
&\leq 2\kappa^{2r}
\end{align*}
\end{itemize}

Plugging the bounds of $I$ and $II$ into \eqref{fbound} leads to 
\begin{align*}
&\|f_\lambda^*\|_\infty\leq   2 \kappa^{2r+1} R D \,\lambda^{-(\frac{1}{2}-r)^+},\\
&\|f_\lambda^*\|_{\mathcal{H}_M}\leq   2 \kappa^{2r} R D \,\lambda^{-(\frac{1}{2}-r)^+}.
\end{align*}

\end{proof}

\begin{proposition}
\label{OPbound1}
Let $\mathcal{H}$ be a separable Hilbert space and let $A$ and $B$  be two bounded self-adjoint positive linear operators on $\mathcal{H}$ and $\lambda>0$. Then

$$
\left\|A_\lambda^{-\frac{1}{2}}B_\lambda ^{\frac{1}{2}}\right\| \leq(1-c)^{-\frac{1}{2}}, \quad \left\|A_\lambda ^{\frac{1}{2}}B_\lambda ^{-\frac{1}{2}}\right\| \leq (1+c)^{\frac{1}{2}}
$$
with
$$
c=\left\|B_\lambda ^{-\frac{1}{2}}(A-B)B_\lambda ^{-\frac{1}{2}}\right\|.
$$
\end{proposition}

\begin{proof}
The proof for the first inequality can for example be found in \cite{features} (Proposition 8). Using simple calculations the second inequality follows from 
\begin{align*}
\left\|(A+\lambda I)^{\frac{1}{2}}(B+\lambda I)^{-\frac{1}{2}}\right\|^2&=\left\|(B+\lambda I)^{-\frac{1}{2}}(A+\lambda I)(B+\lambda I)^{-\frac{1}{2}}\right\|\\
&\leq \left\|(B+\lambda I)^{-\frac{1}{2}}(A-B)(B+\lambda I)^{-\frac{1}{2}}\right\|+\|I\|\leq 1+c\\
\end{align*}

\end{proof}

\begin{proposition}[\cite{features} (Lemma 9)]
\label{OPbound3}
For any $M\geq8\kappa^4\|\mathcal{L}_\infty\|^{-1}\log^2 \frac{2}{\delta}$  we have with probability at least $1-\delta$
$$
\|\mathcal{L}_M\|\geq\frac{1}{2}\|\mathcal{L}_\infty\|.
$$
\end{proposition}
\begin{proof}
For $M\geq8\kappa^4\|\mathcal{L}_\infty\|^{-1}\log^2 \frac{2}{\delta}$ we have from Proposition \ref{OPbound2} ($E_6$) that with probability at least $1-\delta$, $\left\|\mathcal{L}_\infty-\mathcal{L}_M\right\|_{H S}\leq \frac{1}{2}\|\mathcal{L}_\infty\|$ 
and therefore
$$
\|\mathcal{L}_M\|  \geq \|\mathcal{L}_\infty\|-\left\|\mathcal{L}_\infty-\mathcal{L}_M\right\|_{H S}\geq \frac{1}{2} \|\mathcal{L}_\infty\|.
$$
\end{proof}

\begin{proposition}
\label{OPbound5}
Providing Assumption \ref{ass:kernel} 
we have for any $M\geq \frac{8 p\kappa^2 \beta_\infty}{\lambda}$ \\where $\beta_\infty=\log \frac{4 \kappa^2(\mathcal{N}_{\mathcal{L}_\infty}(\lambda)+1)}{\delta\|\mathcal{L}_\infty\|} $ with probability at least $1-\delta$
$$
\left\|\mathcal{L}_{M,\lambda}^{-\frac{1}{2}}\mathcal{L}_{\infty,\lambda}^{\frac{1}{2}}\right\| \leq 2, \quad \left\|\mathcal{L}_{M,\lambda}^{\frac{1}{2}}\mathcal{L}_{\infty,\lambda}^{-\frac{1}{2}}\right\| \leq 2.
$$

\end{proposition}

\begin{proof}
From Proposition \ref{OPbound2} ($E_2$) we have for any $\lambda>0$ ,
\begin{align}
\left\|\mathcal{L}_{\infty,\lambda}^{-\frac{1}{2}}(\mathcal{L}_M-\mathcal{L}_\infty)\mathcal{L}_{\infty,\lambda}^{-\frac{1}{2}}\right\|\leq \frac{4 \kappa^2 \beta_\infty}{3M \lambda}+\sqrt{\frac{2 p\kappa^2 \beta_\infty}{M\lambda}}.
\end{align}
From $M\geq \frac{8 p\kappa^2 \beta_\infty}{\lambda}$ we therefore obtain
\begin{align}
\left\|\mathcal{L}_{\infty,\lambda}^{-\frac{1}{2}}(\mathcal{L}_M-\mathcal{L}_\infty)\mathcal{L}_{\infty,\lambda}^{-\frac{1}{2}}\right\|\leq \frac{3}{4}
\end{align}

The result now follows from Proposition \ref{OPbound1}
\end{proof}

\begin{proposition}
\label{OPbound6}
Providing Assumption \ref{ass:kernel} we have for  any $n\geq \frac{8\kappa^2 \tilde\beta}{\lambda}$ with\\ $\tilde{\beta}:= \log \frac{4 \kappa^2(\left(1+2\log\frac{2}{\delta}\right)4\mathcal{N}_{\mathcal{L}_{\infty}}(\lambda)+1)}{\delta\|\mathcal{L}_\infty\|}$ and $M\geq \frac{8 p\kappa^2 \beta_\infty}{\lambda}\vee 8\kappa^4\|\mathcal{L}_\infty\|^{-1}\log^2 \frac{2}{\delta}$, where $\beta_\infty=\log \frac{4 \kappa^2(\mathcal{N}_{\mathcal{L}_\infty}(\lambda)+1)}{\delta\|\mathcal{L}_\infty\|} $  that with probability at least $1-4\delta$
$$ 
\left\|\widehat{\Sigma}_{M,\lambda}^{-\frac{1}{2}}\Sigma_{M,\lambda}^{\frac{1}{2}}\right\| \leq 2, \quad  \left\|\widehat{\Sigma}_{M,\lambda}^{\frac{1}{2}}\Sigma_{M,\lambda}^{-\frac{1}{2}}\right\| \leq 2 .
$$

\end{proposition}

\begin{proof}

From Proposition \ref{OPbound2} ($E_1$) we have for any $\lambda>0$  with probability at least $1-\delta$,
\begin{align}
\left\|\Sigma_{M,\lambda}^{-\frac{1}{2}}\left(\widehat{\Sigma}_{M}-\Sigma_{M}\right) \Sigma_{M,\lambda}^{-\frac{1}{2}}\right\|&\leq\frac{4 \kappa^2 \beta_M}{3n \lambda}+\sqrt{\frac{2 \kappa^2 \beta_M}{n\lambda}},\label{e2bb}
\end{align}
with $\beta_M=\log \frac{4 \kappa^2(\mathcal{N}_{\mathcal{L}_M}(\lambda)+1)}{\delta\|\mathcal{L}_M\|}$ . 
For $M\geq \frac{8 p\kappa^2 \beta_\infty}{\lambda}$ we obtain from Proposition \ref{prop:effecdim2} that with probability at least $1-2\delta$,
\begin{align}
\mathcal{N}_{\mathcal{L}_{M}}(\lambda)\leq  \left(1+2\log\frac{2}{\delta}\right)4\mathcal{N}_{\mathcal{L}_{\infty}}(\lambda). \label{boundE24}
\end{align}
From Proposition \ref{OPbound3} we have with probability $1-\delta,$

\begin{align}
\|\mathcal{L}_M\|\geq\frac{1}{2}\|\mathcal{L}_\infty\|. \label{boundE56}
\end{align}

Note that the bounds of \eqref{boundE24} and \eqref{boundE56} imply $\beta_M\leq\tilde{\beta}= \log \frac{4 \kappa^2(\left(1+2\log\frac{2}{\delta}\right)4\mathcal{N}_{\mathcal{L}_{\infty}}(\lambda)+1)}{\delta\|\mathcal{L}_\infty\|}$ . Using this together with $n\geq \frac{8\kappa^2 \tilde\beta}{\lambda}$ we obtain for \eqref{e2bb}

\begin{align}
\left\|\Sigma_{M,\lambda}^{-\frac{1}{2}}\left(\widehat{\Sigma}_{M}-\Sigma_{M}\right) \Sigma_{M,\lambda}^{-\frac{1}{2}}\right\|&\leq\frac{4 \kappa^2 \beta_M}{3n \lambda}+\sqrt{\frac{2 \kappa^2 \beta_M}{n\lambda}}\\
&\leq\frac{4 \kappa^2 \tilde\beta}{3n \lambda}+\sqrt{\frac{2 \kappa^2 \tilde\beta}{n\lambda}}\leq \frac{3}{4}.
\end{align}
Note that from Proposition \ref{conditioning} we have that the above inequality holds with probability at least $1-4\delta$.
The result now follows from Proposition \ref{OPbound1}
\end{proof}

\begin{proposition}
\label{OPbound7}
Providing Assumption \ref{ass:kernel} we have for any  
\begin{align*}
M\geq 
\begin{cases}
\frac{8 p\kappa^2 \beta_\infty}{\lambda} & r\in\left(0,\frac{1}{2}\right)\\
\frac{(8 p\kappa^2 \beta_\infty)\vee C_1^{\frac{1}{r}}}{\lambda}\vee \frac{C_2}{\lambda^{1+b(2r-1)} } & r\in\left[\frac{1}{2},1\right] \\
\frac{C_3}{\lambda^{2r}} & r \in(1,\infty)
\end{cases}
\end{align*}
with probability at least $1-3\delta$,
$$
\left\|\mathcal{L}_{M,\lambda}^{-(r\vee1)}\mathcal{L}_{\infty,\lambda}^{r}\right\| \leq \frac{3}{\lambda^{(1-r)^+}},
$$
where  $C_1=2(4\kappa\log\frac{2}{\delta})^{2r-1}(8p\kappa^2\beta_\infty)^{1-r}$ ,  $C_2=4(4c_b\kappa^2\log\frac{2}{\delta})^{2r-1}(8p\kappa^2\beta_\infty)^{2-2r}$, \\ $C_3:= 4\kappa^4C_{\kappa,r}^2\log^2\frac{2}{\delta}$ and with $C_{\kappa,r}$ from Proposition \ref{ineq1}.
\end{proposition}

\begin{proof} For the proof we need to differ between the following three cases:
 \begin{itemize}

 \item CASE ($r\leq \frac{1}{2}$) :  From Proposition \ref{OPbound5} we have with probability at least $1-\delta$,
\begin{align*}
\left\|\mathcal{L}_{M,\lambda}^{-(r\vee1)}\mathcal{L}_{\infty,\lambda}^{r}\right\|&=\left\|\mathcal{L}_{M,\lambda}^{-1}\mathcal{L}_{\infty,\lambda}^{r}\right\| \\
&\leq \lambda^{r-1}\left\|\mathcal{L}_{M,\lambda}^{-r}\mathcal{L}_{\infty,\lambda}^{r}\right\|\\
&\leq \lambda^{r-1}\left\|\mathcal{L}_{M,\lambda}^{-\frac{1}{2}}\mathcal{L}_{\infty,\lambda}^{\frac{1}{2}}\right\|^{2r} \leq 2^{2r}\lambda^{r-1}\leq 3\lambda^{r-1}.
\end{align*}

\item CASE ($r\in[\frac{1}{2},1]$) :  Using $\left\|\mathcal{L}_{\infty,\lambda}^{-1}\mathcal{L}_{\infty,\lambda}^{r}\right\|\leq \lambda^{r-1}$ we have

\begin{align}
\left\|\mathcal{L}_{M,\lambda}^{-(r\vee1)}\mathcal{L}_{\infty,\lambda}^{r}\right\|&=\left\|\mathcal{L}_{M,\lambda}^{-1}\mathcal{L}_{\infty,\lambda}^{r}\right\| \\
&\leq\left\|\left(\mathcal{L}_{M,\lambda}^{-1}-\mathcal{L}_{\infty,\lambda}^{-1}\right)\mathcal{L}_{\infty,\lambda}^{r}\right\|+\lambda^{r-1}.\label{OP8i}
\end{align}

For the norm of the last inequality we have from the algebraic identity 
\\$A^{-1}-B^{-1}=A^{-1}(A-B)B^{-1}$:
\begin{align*}
\left\|\left(\mathcal{L}_{M,\lambda}^{-1}-\mathcal{L}_{\infty,\lambda}^{-1}\right)\mathcal{L}_{\infty,\lambda}^{r}\right\|
=\left\|\mathcal{L}_{M,\lambda}^{-1}\left(\mathcal{L}_{M,\lambda}-\mathcal{L}_{\infty,\lambda}\right)\mathcal{L}_{\infty,\lambda}^{r-1}\right\|
\end{align*}
and from Proposition \ref{OPbound5} we further have with probability at least $1-\delta$,
\begin{align*}
&\left\|\mathcal{L}_{M,\lambda}^{-1}\left(\mathcal{L}_{M,\lambda}-\mathcal{L}_{\infty,\lambda}\right)\mathcal{L}_{\infty,\lambda}^{r-1}\right\|\\
&\leq \lambda^{-\frac{1}{2}}\left\|\mathcal{L}_{M,\lambda}^{-\frac{1}{2}}\mathcal{L}_{\infty,\lambda}^{\frac{1}{2}}\right\| \left\|\mathcal{L}_{\infty,\lambda}^{-\frac{1}{2}}\left(\mathcal{L}_{M,\lambda}-\mathcal{L}_{\infty,\lambda}\right)\mathcal{L}_{\infty,\lambda}^{r-1}\right\|\\
&\leq 2\lambda^{-\frac{1}{2}} \left\|\mathcal{L}_{\infty,\lambda}^{-\frac{1}{2}}\left(\mathcal{L}_{M,\lambda}-\mathcal{L}_{\infty,\lambda}\right)\mathcal{L}_{\infty,\lambda}^{r-1}\right\|.\\
\end{align*}
Since $\sigma:=2-2r\leq 1$ we have from Proposition \ref{ineq3}  

\begin{align*}
& \left\|\mathcal{L}_{\infty,\lambda}^{-\frac{1}{2}}\left(\mathcal{L}_{M,\lambda}-\mathcal{L}_{\infty,\lambda}\right)\mathcal{L}_{\infty,\lambda}^{r-1}\right\|\\
 &\leq \left\|\mathcal{L}_{\infty,\lambda}^{-\frac{1}{2}}\left(\mathcal{L}_{M,\lambda}-\mathcal{L}_{\infty,\lambda}\right)\right\|^{2r-1}  \left\|\mathcal{L}_{\infty,\lambda}^{-\frac{1}{2}}\left(\mathcal{L}_{M,\lambda}-\mathcal{L}_{\infty,\lambda}\right)\mathcal{L}_{\infty,\lambda}^{-\frac{1}{2}}\right\|^{2-2r}.
\end{align*}
Using Proposition \ref{OPbound2} ($E_2$ and $E_5$) we have for the last expression with probability at least $1-2\delta$,
\begin{align*}
\leq \left[\left(\frac{2 
\kappa}{\sqrt{\lambda}M}+\sqrt{\frac{4 \kappa^2 \mathcal{N}_{\mathcal{L}_\infty}(\lambda) }{M}}\right)\log \frac{2}{\delta}\right]^{2r-1}
\left(\frac{4 \kappa^2 \beta_\infty}{3M \lambda}+\sqrt{\frac{2 p\kappa^2 \beta_\infty}{M\lambda}}\right)^{2-2r}
\end{align*}
with $\beta_\infty=\log \frac{4 \kappa^2(\mathcal{N}_{\mathcal{L}_\infty}(\lambda)+1)}{\delta\|\mathcal{L}_\infty\|}$. Using this together with  $M\geq \frac{8 p\kappa^2 \beta_\infty}{\lambda} $ and the simple inequality $(a+b)^{2r-1}\leq a^{2r-1} + b^{2r-1}$ we have
\begin{align*}
&\left\|\left(\mathcal{L}_{M,\lambda}^{-1}-\mathcal{L}_{\infty,\lambda}^{-1}\right)\mathcal{L}_{\infty,\lambda}^{r}\right\|\\
&\leq 2\lambda^{-\frac{1}{2}}\left(\frac{4 
\kappa \log \frac{2}{\delta}}{\sqrt{\lambda}M}+\sqrt{\frac{4 \kappa^2 \mathcal{N}_{\mathcal{L}_\infty}(\lambda) \log \frac{2}{\delta}}{M}}\right)^{2r-1}
\left(\frac{4 \kappa^2 \beta_\infty}{3M \lambda}+\sqrt{\frac{2 p\kappa^2 \beta_\infty}{M\lambda}}\right)^{2-2r}\\
& \leq2\lambda^{-\frac{1}{2}} \left(\frac{4 
\kappa \log \frac{2}{\delta}}{\sqrt{\lambda}M}+\sqrt{\frac{4 \kappa^2 \mathcal{N}_{\mathcal{L}_\infty}(\lambda) \log \frac{2}{\delta}}{M}}\right)^{2r-1}
\left(2\sqrt{\frac{2p\kappa^2 \beta_\infty}{M\lambda}}\right)^{2-2r}\\
&\leq\frac{C_1}{\lambda M^r}+\sqrt{\frac{C_2'\mathcal{N}_{\mathcal{L}_\infty}(\lambda)^{2r-1}}{M\lambda^{3-2r}}}\leq\frac{C_1}{\lambda M^r}+\sqrt{\frac{C_2}{M\lambda^{3-2r+b(2r-1)}}},
\end{align*}
where we used in the last inequality the assumption $\mathcal{N}_{\mathcal{L}_\infty}(\lambda)\leq c_b \lambda^{-b}$  and set \\$C_1=2(4\kappa\log\frac{2}{\delta})^{2r-1}(8p\kappa^2\beta_\infty)^{1-r}$ ,  $C_2=4(4c_b\kappa^2\log^2\frac{2}{\delta})^{2r-1}(8p\kappa^2\beta_\infty)^{2-2r}$.
From $M\geq \frac{C_1^{\frac{1}{r}}}{\lambda} $ and $M\geq \frac{C_2}{\lambda^{1+b(2r-1)} }$ we obtain 
\begin{align*}
\left\|\left(\mathcal{L}_{M,\lambda}^{-1}-\mathcal{L}_{\infty,\lambda}^{-1}\right)\mathcal{L}_{\infty,\lambda}^{r}\right\|
\leq\frac{C_1}{\lambda M^r}+\sqrt{\frac{C_2}{M\lambda^{3-2r+b(2r-1)}}}
\leq 2\lambda^{r-1}.
\end{align*}
Plugging this bound into \eqref{OP8i} leads to
\begin{align*}
\left\|\mathcal{L}_{M,\lambda}^{-(r\vee1)}\mathcal{L}_{\infty,\lambda}^{r}\right\|\leq3 \lambda^{r-1}.
\end{align*}

\item CASE $(r\geq 1)$ :  
\begin{align*}
\left\|\mathcal{L}_{M,\lambda}^{-(r\vee1)}\mathcal{L}_{\infty,\lambda}^{r}\right\|&=\left\|\mathcal{L}_{M,\lambda}^{-r}\mathcal{L}_{\infty,\lambda}^{r}\right\|\\
&\leq 1+ \left\|\mathcal{L}_{M,\lambda}^{-r}\left(\mathcal{L}_{\infty,\lambda}^{r}-\mathcal{L}_{M,\lambda}^{r}\right)\right\|\\
&\leq 1+ \lambda^{-r}C_{\kappa, r}\left\|\mathcal{L}_{\infty,\lambda}-\mathcal{L}_{M,\lambda}\right\|,
\end{align*}
where $C_{\kappa, r}$ is defined in Proposition \ref{ineq1}.
From the bound of Proposition \ref{OPbound2} ($E_6$) we therefore obtain
\begin{align*}
&\left\|\mathcal{L}_{M,\lambda}^{-(r\vee1)}\mathcal{L}_{\infty,\lambda}^{r}\right\|\\
&\leq 1+ \lambda^{-r}C_{1, r}  \left(\frac{2 \kappa^2}{M} + \frac{2 \kappa^2}{\sqrt{M}} \right)\log \frac{2}{\delta}\leq 3
\end{align*}
where used $M\geq C_3 \lambda^{-2r}$, with $C_3:= 4\kappa^4C_{1,r}^2\log^2\frac{2}{\delta}$.
\end{itemize} 
\end{proof}

\begin{proposition}
\label{OPbound8}
  For any $q>0$, $n\geq \max\{8C_{\kappa,q}^2 \kappa^4\lambda^{-2q} \log^2\frac{2}{\delta},\, 100\kappa^2 \mathcal{N}_{\mathcal{L}_{\infty}}(\lambda)\lambda^{-1} \log^3 \frac{2}{\delta}\}$ and $M\geq \frac{8 p\kappa^2 \beta_\infty}{\lambda}$ we have with probability at least $1-\delta$,

\begin{align*}
\left\|\widehat{\Sigma}_{M,\lambda}^{-q} \Sigma_{M,\lambda}^{q}\right\|\leq 2.
\end{align*}
\end{proposition}

\begin{proof}
\begin{itemize}
\item \textbf{Case $q<1$:}

From Proposition \ref{OPbound2}($E_3$) we obtain with probability at least $1-\delta$,
\begin{align*}
\left\|\widehat{\Sigma}_{M,\lambda}^{-q} \Sigma_{M,\lambda}^{q}\right\|&=\left\|\widehat{\Sigma}_{M,\lambda}^{-1} \Sigma_{M,\lambda}\right\|^q\\
&\leq\left\|\widehat{\Sigma}_{M,\lambda}^{-1}\left(\widehat{\Sigma}_{M}-\Sigma_{M}\right)\right\|_{HS}+1\\
&\leq\frac{1}{\sqrt{\lambda}} \left\|\widehat{\Sigma}_{M,\lambda}^{-\frac{1}{2}} \Sigma_{M,\lambda}^{\frac{1}{2}}\right\|\left\|\Sigma_{M,\lambda}^{-\frac{1}{2}}\left(\widehat{\Sigma}_{M}-\Sigma_{M}\right)\right\|_{HS}+1\\
&\leq \frac{1}{\sqrt{\lambda}} \left\|\widehat{\Sigma}_{M,\lambda}^{-\frac{1}{2}} \Sigma_{M,\lambda}^{\frac{1}{2}}\right\| \left(\frac{2\kappa}{\sqrt{\lambda} n}+\sqrt{\frac{ 4\kappa^2 \mathcal{N}_{\mathcal{L}_M}(\lambda) }{ n}}\right)\log \frac{2}{\delta}+1.
\end{align*}
We have from Proposition \ref{OPbound6} with probability at least $1-\delta$,
$$ 
\left\|\widehat{\Sigma}_{M,\lambda}^{-\frac{1}{2}}\Sigma_{M,\lambda}^{\frac{1}{2}}\right\| \leq 2
$$
and therefore
\begin{align}
\left\|\widehat{\Sigma}_{M,\lambda}^{-q} \Sigma_{M,\lambda}^{q}\right\|\leq \frac{2}{\sqrt{\lambda}} \left(\frac{2\kappa}{\sqrt{\lambda} n}+\sqrt{\frac{ 4\kappa^2 \mathcal{N}_{\mathcal{L}_M}(\lambda) }{ n}}\right)\log \frac{2}{\delta}+1.\label{ineqT2ii}
\end{align}

From  \ref{prop:effecdim2} we have with probability at least $1-2\delta$,
$$
\mathcal{N}_{\mathcal{L}_{M}}(\lambda)\leq  \left(1+2\log\frac{2}{\delta}\right)4\mathcal{N}_{\mathcal{L}_{\infty}}(\lambda).
$$
Plugging this bound into \eqref{ineqT2ii} leads to
\begin{align}
\left\|\widehat{\Sigma}_{M,\lambda}^{-q} \Sigma_{M,\lambda}^{q}\right\|\leq \frac{2}{\sqrt{\lambda}} \left(\frac{2\kappa}{\sqrt{\lambda} n}+\sqrt{\frac{ 4\kappa^2 \left(1+2\log\frac{2}{\delta}\right)4\mathcal{N}_{\mathcal{L}_{\infty}}(\lambda) }{ n}}\right)\log \frac{2}{\delta}+1\leq 2
\end{align}

where we used $n\geq 100\kappa^2 \mathcal{N}_{\mathcal{L}_{\infty}}(\lambda)\lambda^{-1} \log^3 \frac{2}{\delta}$ in the last inequality.

\item \textbf{Case $q\geq1$:}   From Proposition \ref{ineq1} and Proposition \ref{OPbound2}($E_7$) we have with probability at least $1-\delta$,

\begin{align*}
\left\|\widehat{\Sigma}_{M,\lambda}^{-q} \Sigma_{M,\lambda}^q\right\|
&\leq\lambda^{-q}\left\|\widehat{\Sigma}_{M}^q-\Sigma_{M}^q\right\|_{HS}+1\\
&\leq  \lambda^{-q}C_{\kappa,q}\left\|\widehat{\Sigma}_{M}-\Sigma_{M}\right\|_{HS}+1\\
&\leq  \lambda^{-q}C_{\kappa,q}\left(\frac{2 \kappa^2}{n} + \frac{2 \kappa^2}{\sqrt{n}} \right)\log \frac{2}{\delta}+1\leq2
\end{align*}
where we used $n\geq 8C_{\kappa,r}^2 \kappa^4\lambda^{-2q} \log^2\frac{2}{\delta}$ for the last inequality.
\end{itemize}
\end{proof}

%%%%%%%%%%%%%%%%%%%%%%%%%%%%%%%%%%%%%%%%%%%%%%%%%%%%%%%%%%%%%%%%%%%%%%%%%%%%%%%%%%%%%%%%%%%%%%%%%%%%%%%%%%%%%%%%%%SEFFECTIVE DIMENSION BOUNDS %%%%%%%%%%%%%%%%%%%%%%%%%%%%%%%%%%%%%%%%%%%%%%%%%%%%%%%%%%%%%%%%%%%%%%%%%%%%%%%%%%%%%%%%%%%%%%%%%%%%%%%%%%%%%%%%%%

\begin{proposition}
\label{prop:effecdim2}
For any $M\geq \frac{8 p\kappa^2 \beta_\infty}{\lambda}$ we have with probability at least $1-2\delta$,
$$
\mathcal{N}_{\mathcal{L}_{M}}(\lambda)\leq  \left(1+2\log\frac{2}{\delta}\right)4\mathcal{N}_{\mathcal{L}_{\infty}}(\lambda).
$$

\end{proposition}

\begin{proof}
\begin{align*}
\mathcal{N}_{\mathcal{L}_{M}}(\lambda)&\leq \text{ Tr}[\mathcal{L}_M\mathcal{L}_{\infty,\lambda}^{-1}]\left\|\mathcal{L}_{\infty,\lambda}^{\frac{1}{2}}\mathcal{L}_{M,\lambda}^{-\frac{1}{2}}\right\|^2\\
&=\left(\mathcal{N}_{\mathcal{L}_\infty}+\text{ Tr}[(\mathcal{L}_M-\mathcal{L}_\infty)\mathcal{L}_{\infty,\lambda}^{-1}]\right)\left\|\mathcal{L}_{\infty,\lambda}^{\frac{1}{2}}\mathcal{L}_{M,\lambda}^{-\frac{1}{2}}\right\|^2\\
&=\left(\mathcal{N}_{\mathcal{L}_\infty}+\|B\|_{HS}\right)\left\|\mathcal{L}_{\infty,\lambda}^{\frac{1}{2}}\mathcal{L}_{M,\lambda}^{-\frac{1}{2}}\right\|^2,
\end{align*}

where $B:=\mathcal{L}_{\infty,\lambda}^{-\frac{1}{2}}(\mathcal{L}_M-\mathcal{L}_\infty)\mathcal{L}_{\infty,\lambda}^{-\frac{1}{2}}$.  Proposition \ref{OPbound2}($E_4$) we have with probability at least $1-\delta$,
$$
\|B\|_{HS}\leq   2\left(\frac{2\kappa^2}{\lambda M}+\sqrt{\frac{ \kappa^2 \mathcal{N}_{\mathcal{L}_\infty}(\lambda) }{\lambda M}}\right)\log \frac{2}{\delta}.
$$
Using $\lambda> 4\kappa^2M^{-1}$ we obtain
$$
\|B\|_{HS}\leq 2\mathcal{N}_{\mathcal{L}_\infty}(\lambda) \log \frac{2}{\delta}
$$
Further we have from Proposition \ref{OPbound5} with probability at least $1-\delta$,

$$
\left\|\mathcal{L}_{\infty,\lambda}^{\frac{1}{2}}\mathcal{L}_{M,\lambda}^{-\frac{1}{2}}\right\|^2\leq4.
$$

To sum up, we obtain 
\begin{align*}
\mathcal{N}_{\mathcal{L}_{M}}(\lambda)\leq\left(\mathcal{N}_{\mathcal{L}_\infty}+\|B\|_{HS}\right)\left\|\mathcal{L}_{\infty,\lambda}^{\frac{1}{2}}\mathcal{L}_{M,\lambda}^{-\frac{1}{2}}\right\|^2\leq \left(1+2\log\frac{2}{\delta}\right)4\mathcal{N}_{\mathcal{L}_{\infty}}(\lambda).
\end{align*}

\end{proof}

%%%%%%%%%%%%%%%%%%%%%   KONZENTRATIONS UNGLEICHUNGEN   %%%%%%%%%%%%%%
%%%%%%%%%%%%%%%%%%%%%%%%%%%%%%%%%%%%%%%%%%%%%%%%%%%%%%%%%
%%%%%%%%%%%%%%%%%%%%%%%%%%%%%%%%%%%%%%%%%%%%%%%%%%%%%%%%%%%%%%%%%%%%%%%%%%%%%%%%%%%%%%%%%%%%%%%%%%%%%%%%%%%%%%%%%%%%%%%%%%%%%%%%%%%%%%%%%%%%%%%%%%%%%%%%%%%%%%%%%%%%%%%%%%%
\subsection{Appendix III}
\label{III}
\input{K.U}

%%%%%%%%%%%%%%%%%%%%%%%%%%TRASH%%%%%%%%%%%%%%%%%%%%%%%%%%%%%%%%%%%%%%%%%%%%%%%%%%%%%%%%%%%%%%%%%%%%%%%%%%%%%%%%%%%%%%%%%%%%%%%%%%%%%%%%%%%%%%%%%%%%%%%%%%%%%%%%%%%%%%%%%%%%%%%%%%%%%%%%%%%%%%%%%%%%%%%%%%%%%%%%%%%%%%%%%%%%%%%%%%%%%%%%%%%%%%%%%%%%%%%%%%%%%%%%%%%%%%%%%%%%%%%%%%%%%%%%%%%%%%%%%%%%%%%

%% file: K.U.tex
%%%%%%%%%%%%%%%%%%%%%   KONZENTRATIONS UNGLEICHUNGEN   %%%%%%%%%%%%%%
%%%%%%%%%%%%%%%%%%%%%%%%%%%%%%%%%%%%%%%%%%%%%%%%%%%%%%%%%

\begin{proposition}
\label{OPbound0}
Let $\mathcal{X}_1, \cdots, \mathcal{X}_m$ be a sequence of independently and identically distributed selfadjoint Hilbert-Schmidt operators on a separable Hilbert space. Assume that $\mathbb{E}\left[\mathcal{X}_1\right]=0$, and $\left\|\mathcal{X}_1\right\| \leq B$ almost surely for some $B>0$. Let $\mathcal{V}$ be a positive trace-class operator such that $\mathbb{E}\left[\mathcal{X}_1^2\right] \preccurlyeq \mathcal{V}$. Then with probability at least $1-\delta,(\delta \in] 0,1[)$, there holds
$$
\left\|\frac{1}{m} \sum_{i=1}^m \mathcal{X}_i\right\| \leq \frac{2 B \beta}{3 m}+\sqrt{\frac{2\|\mathcal{V}\| \beta}{m}}, \quad \beta=\log \frac{4 \operatorname{tr} \mathcal{V}}{\|\mathcal{V}\| \delta}
$$
\end{proposition}
\begin{proof}
The proposition was first established for matrices by  \cite{Tropp_2011}. For the general case including operators the proof can for example be found in  \cite{spectral.rates} (see Lemma 26).
\end{proof}

\begin{proposition}
\label{concentrationineq0}
The following concentration result for Hilbert space valued random variables can be found in (Caponnetto and De Vito, 2007 \cite{Caponetto}).\\
\\
Let $w_{1}, \cdots, w_{n}$ be i.i.d random variables in a separable Hilbert space with norm $\|.\|$. Suppose that there are two positive constants $B$ and $\sigma^{2}$ such that
\begin{align}
\mathbb{E}\left[\left\|w_{1}-\mathbb{E}\left[w_{1}\right]\right\|^{l}\right] \leq \frac{1}{2} l ! B^{l-2} V^{2}, \quad \forall l \geq 2 \label{cons}
\end{align}
Then for any $0<\delta<1 / 2$, the following holds with probability at least $1-\delta$,
$$
\left\|\frac{1}{n} \sum_{k=1}^{n} w_{n}-\mathbb{E}\left[w_{1}\right]\right\| \leq \left(\frac{2B}{n}+\frac{2V}{\sqrt{n}}\right) \log \frac{2}{\delta} .
$$
In particular, (\ref{cons}) holds if
$$
\left\|w_{1}\right\| \leq B / 2 \quad \text { a.s., } \quad \text { and } \quad \mathbb{E}\left[\left\|w_{1}\right\|^{2}\right] \leq V^{2} .
$$
\end{proposition}

\begin{proposition}
\label{OPbound2}
For any $\lambda>0$ define the following events,\vspace{0.4cm}\\

$\begin{aligned}
&E_1=\left\{\left\|\Sigma_{M,\lambda}^{-\frac{1}{2}}\left(\widehat{\Sigma}_{M}-\Sigma_{M}\right) \Sigma_{M,\lambda}^{-\frac{1}{2}}\right\|\leq\frac{4 \kappa^2 \beta_M}{3n \lambda}+\sqrt{\frac{2 \kappa^2 \beta_M}{n\lambda}}\right\},  &&\beta_M=\log \frac{4 \kappa^2(\mathcal{N}_{\mathcal{L}_M}(\lambda)+1)}{\delta\|\mathcal{L}_M\|},\\[7pt]
&E_2=\left\{\left\|\mathcal{L}_{\infty,\lambda}^{-\frac{1}{2}}(\mathcal{L}_M-\mathcal{L}_\infty)\mathcal{L}_{\infty,\lambda}^{-\frac{1}{2}}\right\|\leq \frac{4 \kappa^2 \beta_\infty}{3M \lambda}+\sqrt{\frac{2 p\kappa^2 \beta_\infty}{M\lambda}}\right\}, &&\beta_\infty=\log \frac{4 \kappa^2(\mathcal{N}_{\mathcal{L}_\infty}(\lambda)+1)}{\delta\|\mathcal{L}_\infty\|}  ,\\
&E_3=\left\{\left\|\Sigma_{M,\lambda}^{-\frac{1}{2}}\left(\widehat{\Sigma}_{M}-\Sigma_{M}\right)\right\|_{HS}\leq  \left(\frac{2\kappa}{\sqrt{\lambda} n}+\sqrt{\frac{ 4\kappa^2 \mathcal{N}_{\mathcal{L}_M}(\lambda) }{ n}}\right)\log \frac{2}{\delta}\right\},\\
&E_4=\left\{\left\|\mathcal{L}_{\infty,\lambda}^{-\frac{1}{2}}(\mathcal{L}_M-\mathcal{L}_\infty)\mathcal{L}_{\infty,\lambda}^{-\frac{1}{2}}\right\|_{HS}\leq  \left(\frac{4\kappa^2}{\lambda M}+\sqrt{\frac{ 4\kappa^2 \mathcal{N}_{\mathcal{L}_\infty}(\lambda) }{\lambda M}}\right)\log \frac{2}{\delta}\right\},\\
\end{aligned}$\\
$\begin{aligned}
&E_5=\left\{\left\|\mathcal{L}_{\infty,\lambda}^{-\frac{1}{2}}\left(\mathcal{L}_M-\mathcal{L}_\infty\right)\right\| \leq \left(\frac{2 
\kappa}{\sqrt{\lambda}M}+\sqrt{\frac{4 \kappa^2 \mathcal{N}_{\mathcal{L}_\infty}(\lambda) }{M}}\right)\log \frac{2}{\delta}\right\},\\
&E_6=\left\{\left\|\mathcal{L}_\infty-\mathcal{L}_M\right\|_{H S} \leq \left(\frac{2 \kappa^2}{M} + \frac{2 \kappa^2}{\sqrt{M}} \right)\log \frac{2}{\delta}\right\}\,,\\
&E_7=\left\{\left\|\widehat{\Sigma}_{M}-\Sigma_{M}\right\|_{HS} \leq \left(\frac{2 \kappa^2}{n} + \frac{2 \kappa^2}{\sqrt{n}} \right)\log \frac{2}{\delta}\right\}\,.
\end{aligned}$\vspace{0.4cm}\\
Providing Assumption \ref{ass:input}  we have for any $\delta \in(0,1)$ that each of the above events holds true with probability at least $1-\delta$  .
\end{proposition}
\begin{proof}
The bound for $E_1$  follows exactly the same steps as in the proof of \cite{spectral.rates} (Lemma 18). The events $E_2-E_7$ have been bounded in \cite{features} ( see Proposition 6, Lemma 8 and Proposition 10). However, due to different assumptions and a different setting we attain slightly different bounds and therefore give the proof of  the events $E_2-E_7$ for completeness.\\

\textbf{$E_2)$} First note that $\mathcal{L}_M$ can be expressed by 

$$
\mathcal{L}_M=\frac{1}{M}\sum_{m=1}^M \sum_{i=1}^p \varphi_m^{(i)}\otimes\varphi_m^{(i)},
$$
where $\varphi_m(.)=\varphi(.,\omega_m)$ and the tensor product is taken with respect to the $L^2(\mathcal{X},\rho_x)$ scalar product. The above equality can be checked by simple calculations: 
\begin{align*}
\mathcal{L}_Mg(x) &= \int K_M(x,\tilde{x}) g(\tilde{x})d\rho_x(\tilde{x})  \\
&= \frac{1}{M} \sum_{m=1}^M  \sum_{i=1}^p  \int  \varphi_m^{(i)}(x)\otimes_\mathcal{Y} \varphi_m^{(i)}(\tilde{x})g(\tilde{x})d\rho_x(\tilde{x})  \\
&= \frac{1}{M} \sum_{m=1}^M  \sum_{i=1}^p   \varphi_m^{(i)}(x)  \int \left\langle\varphi_m^{(i)}(\tilde{x}),g(\tilde{x})\right\rangle_\mathcal{Y} d\rho_x(\tilde{x})  \\
&=  \frac{1}{M}\sum_{m=1}^M\sum_{i=1}^p  \left(\varphi_m^{(i)}\otimes \varphi_m^{(i)}\right)(g)(x) \\
\end{align*}

Analog we have $\mathcal{L}_{\infty}=\mathbb{E}[ \sum_{i=1}^p \varphi^{(i)}\otimes\varphi^{(i)}]$.

Now define $\mathcal{X}_m:= \mathcal{L}_{\infty,\lambda}^{-\frac{1}{2}}(\mathcal{L}_{M}^{(m)} -\mathcal{L}_{\infty})\mathcal{L}_{\infty,\lambda}^{-\frac{1}{2}}$, with $\mathcal{L}_{M}^{(m)}:=\sum_{i=1}^p\varphi_m^{(i)}\otimes\varphi_m^{(i)}$.
We now obtain
\begin{align*}
\|\mathcal{X}_1\|\leq \left\|\mathcal{L}_{\infty,\lambda}^{-\frac{1}{2}}\mathcal{L}_{M}^{(m)} \mathcal{L}_{\infty,\lambda}^{-\frac{1}{2}}\right\| + \mathbb{E}\left\|\mathcal{L}_{\infty,\lambda}^{-\frac{1}{2}}\mathcal{L}_{M}^{(m)} \mathcal{L}_{\infty,\lambda}^{-\frac{1}{2}}\right\|\leq 2\frac{\kappa^2}{\lambda}:=B,
\end{align*}
where we used for the last inequality
$$
 \left\|\mathcal{L}_{\infty,\lambda}^{-\frac{1}{2}}\mathcal{L}_{M}^{(m)} \mathcal{L}_{\infty,\lambda}^{-\frac{1}{2}}\right\|\leq \lambda^{-1} \left\|\mathcal{L}_{M}^{(m)} \right\|\leq \frac{\kappa^2}{\lambda}.
$$
For the second moment we have from Jensen-inequality

\begin{align*}
\mathbb{E}\left[\mathcal{X}^2\right] &\preccurlyeq \mathbb{E}\left[\left(\mathcal{L}_{\infty,\lambda}^{-\frac{1}{2}}\mathcal{L}_{M}^{(m)} \mathcal{L}_{\infty,\lambda}^{-\frac{1}{2}}\right)^2\right]\\
&\preccurlyeq  \mathbb{E}\left[p\sum_{i=1}^p\left(\mathcal{L}_{\infty,\lambda}^{-\frac{1}{2}}\varphi_m^{(i)}\otimes\varphi_m^{(i)} \mathcal{L}_{\infty,\lambda}^{-\frac{1}{2}}\right)^2\right]\\
&=  \mathbb{E}\left[p\sum_{i=1}^p\left\|\mathcal{L}_{\infty,\lambda}^{-\frac{1}{2}}\varphi_m^{(i)}\right\|^2_{L^2_{\rho_x}}\mathcal{L}_{\infty,\lambda}^{-\frac{1}{2}}\varphi_m^{(i)}\otimes\varphi_m^{(i)} \mathcal{L}_{\infty,\lambda}^{-\frac{1}{2}}\right]\\
&\preccurlyeq \mathbb{E}\left[p\frac{\kappa^2}{\lambda}\mathcal{L}_{\infty,\lambda}^{-\frac{1}{2}}\mathcal{L}^{(m)}_{M}\mathcal{L}_{\infty,\lambda}^{-\frac{1}{2}}\right]\\
&= \frac{p\kappa^2}{\lambda}\mathcal{L}_{\infty}\mathcal{L}_{\infty,\lambda}^{-1}:=\mathcal{V}
\end{align*}

For $\beta=\log \frac{4 \operatorname{tr} \mathcal{V}}{\|\mathcal{V}\| \delta}$ we have
\begin{align*}
\beta&=\log \frac{4 \mathcal{N}_{\mathcal{L}_\infty}(\lambda)}{\|\mathcal{L}_{\infty}\mathcal{L}_{\infty,\lambda}^{-1}\| \delta}\\
&=\log \frac{4 \mathcal{N}_{\mathcal{L}_\infty}(\lambda)(\|\mathcal{L}_\infty\|+\lambda)}{\|\mathcal{L}_\infty\| \delta}\\[7pt]
&\leq \log \frac{4 \mathcal{N}_{\mathcal{L}_\infty}(\lambda)\|\mathcal{L}_\infty\|+4\operatorname{tr}\mathcal{L}_\infty}{\|\mathcal{L}_\infty\| \delta}\leq \log \frac{4\kappa^2 (\mathcal{N}_{\mathcal{L}_\infty}(\lambda)+1)}{\|\mathcal{L}_\infty\|\delta}.
\end{align*}

The claim now follows from Proposition \ref{OPbound0}.\\

\textbf{$E_3)$} Set $w_i:= \Sigma_{M,\lambda}^{-\frac{1}{2}}\xi_i$ with $\xi_i= K_{M,x_i}K_{M,x_i}^*$. Note that $\mathbb{E}[\xi_i]=\Sigma_M$ and 
\begin{align*}
\|w_i\|_{HS}&\leq \left\|\Sigma_{M,\lambda}^{-\frac{1}{2}}K_{M,x_i}K_{M,x_i}^*\right\|_{HS}\\
&\leq \lambda^{-1/2}\left\| K_{M,x_i}K_{M,x_i}^* \right\|_{HS}^2 \\
&= \lambda^{-1/2}\left\| K_{M}(x_i,x_i) \right\|_{HS}^2 \leq\frac{\kappa^2}{\sqrt{\lambda}}=:B
\end{align*}

For the second moment we have, 

\begin{align*}
\mathbb{E}\left\|w_i\right\|^2_{HS}&\leq \kappa^2 \mathbb{E}\|\Sigma_{M,\lambda}^{-\frac{1}{2}}K_{M,x_i}K_{M,x_i}^*\Sigma_{M,\lambda}^{-\frac{1}{2}}\|_{HS} \\
&\leq \kappa^2\mathbb{E}\operatorname{tr}\left[\Sigma_{M,\lambda}^{-\frac{1}{2}}K_{M,x_i}K_{M,x_i}^*\Sigma_{M,\lambda}^{-\frac{1}{2}}\right]= \kappa^2\mathcal{N}_{\mathcal{L}_{M}}(\lambda)=:V^2
\end{align*}

The claim now follows from Proposition \ref{concentrationineq0}.

\textbf{$E_4)$}  Set $w_m:= \mathcal{L}_{\infty,\lambda}^{-\frac{1}{2}}(\mathcal{L}_{M}^{(m)} -\mathcal{L}_{\infty})\mathcal{L}_{\infty,\lambda}^{-\frac{1}{2}}$ . Note that we have
\begin{align*}
\|w_m\|_{HS}&\leq \left\|\mathcal{L}_{\infty,\lambda}^{-\frac{1}{2}}\mathcal{L}_{M}^{(m)} \mathcal{L}_{\infty,\lambda}^{-\frac{1}{2}}\right\|_{HS} + \operatorname{tr}\left[\mathcal{L}_\infty \mathcal{L}_{\infty,\lambda}^{-1}\right]\\
&\leq \left\|\mathcal{L}_{\infty,\lambda}^{-\frac{1}{2}} \left(\sum_{i=1}^p\varphi_m^{(i)}\otimes\varphi_m^{(i)}\right)\mathcal{L}_{\infty,\lambda}^{-\frac{1}{2}}\right\|_{HS} + \mathcal{N}_{\mathcal{L}_\infty}(\lambda)\\
&\leq \lambda^{-1}\sum_{i=1}^p\left\| \varphi_m^{(i)}\otimes\varphi_m^{(i)} \right\|_{HS} + \mathcal{N}_{\mathcal{L}_\infty}(\lambda)\\
&\leq \lambda^{-1}\sum_{i=1}^p\left\| \varphi_m^{(i)} \right\|_{L^2_{\rho_x}}^2 + \mathcal{N}_{\mathcal{L}_\infty}(\lambda)\leq\frac{2\kappa^2}{\lambda}=:B
\end{align*}

For the second moment we have,

\begin{align*}
\mathbb{E}\left\|w_m\right\|_{HS}^2\leq\mathbb{E}\operatorname{tr}\left[\left(\mathcal{L}_{\infty,\lambda}^{-\frac{1}{2}}\mathcal{L}_{M}^{(m)} \mathcal{L}_{\infty,\lambda}^{-\frac{1}{2}}\right)^2\right]\leq \frac{\kappa^2}{\lambda}\mathbb{E}\operatorname{tr}\left[\mathcal{L}_{\infty,\lambda}^{-\frac{1}{2}}\mathcal{L}_{M}^{(m)} \mathcal{L}_{\infty,\lambda}^{-\frac{1}{2}}\right]= \frac{\kappa^2}{\lambda}\mathcal{N}_{\mathcal{L}_{\infty}}(\lambda)=:V^2
\end{align*}
where we used $\|\mathcal{L}_{\infty,\lambda}^{-\frac{1}{2}}\mathcal{L}_{M}^{(m)} \mathcal{L}_{\infty,\lambda}^{-\frac{1}{2}}\|\leq \frac{\kappa^2}{\lambda}$ for the last inequality. 
The claim now follows from Proposition \ref{concentrationineq0}.

\textbf{$E_5)$} Set $w_m:= \mathcal{L}_{\infty,\lambda}^{-\frac{1}{2}}\mathcal{L}_{M}^{(m)}$ . Note that we have
\begin{align*}
\|w_m\|_{HS}&\leq \left\|\mathcal{L}_{\infty,\lambda}^{-\frac{1}{2}}\mathcal{L}_{M}^{(m)} \right\|_{HS} \\
&\leq \left\|\mathcal{L}_{\infty,\lambda}^{-\frac{1}{2}} \left(\sum_{i=1}^p\varphi_m^{(i)}\otimes\varphi_m^{(i)}\right)\right\|_{HS}\\
&\leq \lambda^{-1/2}\sum_{i=1}^p\left\| \varphi_m^{(i)} \right\|_{L^2_{\rho_x}}^2 \leq\frac{\kappa^2}{\sqrt{\lambda}}=:B
\end{align*}

For the second moment we have,

\begin{align*}
\mathbb{E}\left\|w_m\right\|^2_{HS}\leq \kappa^2 \mathbb{E}\|\mathcal{L}_{\infty,\lambda}^{-\frac{1}{2}}\mathcal{L}_{M}^{(m)} \mathcal{L}_{\infty,\lambda}^{-\frac{1}{2}}\|_{HS} \leq \kappa^2\mathbb{E}\operatorname{tr}\left[\mathcal{L}_{\infty,\lambda}^{-\frac{1}{2}}\mathcal{L}_{M}^{(m)} \mathcal{L}_{\infty,\lambda}^{-\frac{1}{2}}\right]= \kappa^2\mathcal{N}_{\mathcal{L}_{\infty}}(\lambda)=:V^2
\end{align*}

The claim now follows from Proposition \ref{concentrationineq0} together with the fact that the operator norm can be bounded by the Hilbert-Schmidt norm: $\|.\|\leq\|.\|_{HS}$ .

\textbf{$E_6)$} Set $w_m:= \mathcal{L}_{M}^{(m)}$ . Note that we have
\begin{align*}
\|w_m\|_{HS}&\leq \left\|\mathcal{L}_{M}^{(m)} \right\|_{HS} = \left\|\sum_{i=1}^p\varphi_m^{(i)}\otimes\varphi_m^{(i)}\right\|_{HS}\\
&\leq \sum_{i=1}^p\left\| \varphi_m^{(i)} \right\|_{L^2_{\rho_x}}^2 \leq \kappa^2=:B
\end{align*}

For the second moment we have,

\begin{align*}
\mathbb{E}\left\|w_m\right\|^2_{HS}\leq \kappa^4 =:V^2
\end{align*}

The claim now follows from Proposition \ref{concentrationineq0}

\textbf{$E_7)$}  Set $w_i:= \xi_i =K_{M,x_i}K_{M,x_i}^*$. Note that 
\begin{align*}
\|w_i\|_{HS}&= \left\|K_{M,x_i}K_{M,x_i}^*\right\|_{HS}\leq \kappa^2 =:B
\end{align*}

For the second moment we have,

\begin{align*}
\mathbb{E}\left\|w_i\right\|_{HS}^2\leq \kappa^4=:V^2
\end{align*}

The claim now follows from Proposition \ref{concentrationineq0}.
\end{proof}

\begin{proposition}
\label{concentrationineq1}
Provided Assumptions \ref{ass:input} we have that the following event holds with probability at least $1-\delta$, 
\begin{align*}
E_8= \left\{\left\|\Sigma_{M,\lambda}^{-\frac{1}{2}}\widehat{\mathcal{S}}_{M}^{*}\left(y-\bar{g}_\rho\right)\right\|_{\mathcal{H}_M} \leq \left(\frac{4QZ\kappa}{\sqrt{\lambda}n}+\frac{4Q\sqrt{\mathcal{N}_{\mathcal{L}_M}(\lambda)}}{\sqrt{n}}\right) \log \frac{2}{\delta}\right\} \,.
\end{align*}

\end{proposition}

\begin{proof}
We want to use Proposition \ref{concentrationineq0} to prove the statement. Therefore define 
\\$w_i:=\Sigma_{M,\lambda}^{-\frac{1}{2}}K_{M,x_i}\left(y_i-g_\rho(x_i)\right)$.  Note that $\mathbb{E}w_i= 0 $ and $\frac{1}{n}\sum_{i=1}^n w_i = \Sigma_{M,\lambda}^{-\frac{1}{2}}\widehat{\mathcal{S}}_{M}^{*}\left(y-\bar{g}_\rho\right)$.
Further we have from Assumption \ref{ass:input},
\begin{align*}
&\mathbb{E}\left[\left\|w\right\|_{\mathcal{H}_M}^{l}\right] \\
&= \int_{\mathcal{X}} \int_{\mathcal{Y}}   \left\|y-g_\rho(x)\right\|_\mathcal{Y}^l \rho(dy|x)\|\Sigma_{M,\lambda}^{-\frac{1}{2}}K_{M,x}\|^l\rho_x(dx)\\
&\leq2^{l-1} \int_{\mathcal{X}} \int_{\mathcal{Y}}   \left(\|y\|_\mathcal{Y}^l+Q^l\right)\rho(dy|x)\|\Sigma_{M,\lambda}^{-\frac{1}{2}}K_{M,x}\|^l\rho_x(dx)\\
&\leq2^{l-1}\left(\frac{1}{2} l ! Z^{l-2} Q^2+Q^l\right) \int_{\mathcal{X}}\|\Sigma_{M,\lambda}^{-\frac{1}{2}}K_{M,x}\|^l\rho_x(dx)\\
&\leq2^{l-1}\left(\frac{1}{2} l ! Z^{l-2} Q^2+Q^l\right) \sup_{x\in\mathcal{X}}\|\Sigma_{M,\lambda}^{-\frac{1}{2}}K_{M,x}\|^{l-2}\int_{\mathcal{X}}tr\left(\Sigma_{M,\lambda}^{-1}K_{M,x_i}K_{M,x_i}^*\right)\rho_x(dx)\\
&\leq2^{l-1}\left(\frac{1}{2} l ! Z^{l-2} Q^2+Q^l\right) \left(\frac{\kappa}{\sqrt{\lambda}}\right)^{l-2}tr\left(\Sigma_{M,\lambda}^{-1}\int_{\mathcal{X}}K_{M,x_i}K_{M,x_i}^*\rho_x(dx)\right)\\
&\leq\frac{1}{2} l ! \left(\frac{2QZ\kappa}{\sqrt{\lambda}}\right)^{l-2}\left(2Q\sqrt{\mathcal{N}_{\mathcal{L}_M}(\lambda)}\right)^2\\
&= \frac{1}{2} l ! B^{l-2} V^{2}.
\end{align*}

Therefore the statement follows from Proposition \ref{concentrationineq0}.
\end{proof}

\begin{proposition}
\label{concentrationineq2}
Provided the assumption $\|g_\rho\|_\infty\leq Q$ and the bound of Proposition \ref{ineq5} : $\|f^*_\lambda\|_\infty \leq  C_{\kappa,R,D} \,\lambda^{-(\frac{1}{2}-r)^+}$,  where $C_{\kappa,R,D}=2 \kappa^{2r+1} R D$. Then the following event holds with probability at least $1-\delta$, 
\begin{align*}
E_9= \left\{\left|\frac{1}{n}\left\|\bar{g}_\rho-\widehat{\mathcal{S}}_M f_\lambda^*\right\|_2^2-\left\|g_\rho-\mathcal{S}_M f_\lambda^*\right\|_{L^2(\rho_x)}^2\right| \leq  2\left(\frac{B_\lambda}{n}+\frac{V_\lambda}{\sqrt{n}}\right) \log \frac{2}{\delta}\right\},
\end{align*}
where $B_\lambda:=4\left(Q^2+ C_{\kappa,R,D}^2 \,\lambda^{-2(\frac{1}{2}-r)^+}\right)$ and $V_\lambda:=\sqrt{2}\left(Q+C_{\kappa,R,D}  \,\lambda^{-(\frac{1}{2}-r)^+}\right)\left\|g_\rho-\mathcal{S}_M f_\lambda^*\right\|_{L^2(\rho_x)} $.
\end{proposition}

\begin{proof}

We want to use Proposition \ref{concentrationineq0} to prove the statement. Therefore define 
\\$w_i:= \left\|g_\rho(x_i)- f_\lambda^*(x_i)\right\|_\mathcal{Y}^2$.  Note that $\mathbb{E}w_1= \left\|g_\rho-\mathcal{S}_M f_\lambda^*\right\|_{L^2(\rho_x)}^2 $ and therefore
$$
\left|\frac{1}{n}\sum_{i=1}^n w_i-\mathbb{E}w_1\right|= \left|\frac{1}{n}\left\|\bar{g}_\rho-\widehat{\mathcal{S}}_M f_\lambda^*\right\|_2^2-\left\|g_\rho-\mathcal{S}_M f_\lambda^*\right\|_{L^2(\rho_x)}^2\right|
$$

It remains to bound $|w_i|$ and $\mathbb{E}w_1^2$. Using the assumption $\|g_\rho\|_\infty:=\sup_{x\in\mathcal{X}}\|g_\rho(x)\|_\mathcal{Y}\leq Q$ and Proposition \ref{ineq5} we have
\begin{align*}
|w_i|\leq 2\left(Q^2+ C_{\kappa,R,D}^2 \,\lambda^{-2(\frac{1}{2}-r)^+}\right)
\end{align*}
and further
\begin{align*}
\mathbb{E}\left[w_1\right] ^2&\leq   2\left(Q^2+ C_{\kappa,R,D}^2 \,\lambda^{-2(\frac{1}{2}-r)^+}\right)\mathbb{E}[w_1]\\&= 2\left(Q^2+C_{\kappa,R,D}^2  \,\lambda^{-2(\frac{1}{2}-r)^+}\right)\left\|g_\rho-\mathcal{S}_M f_\lambda^*\right\|_{L^2(\rho_x)}^2 
\end{align*}
Therefore the statement follows from Proposition \ref{concentrationineq0}.
\end{proof}

%% file: main.bbl
\begin{thebibliography}{47}
\providecommand{\natexlab}[1]{#1}
\providecommand{\url}[1]{\texttt{#1}}
\expandafter\ifx\csname urlstyle\endcsname\relax
  \providecommand{\doi}[1]{doi: #1}\else
  \providecommand{\doi}{doi: \begingroup \urlstyle{rm}\Url}\fi

\bibitem[Aleksandrov and
  Peller(2009)]{aleksandrov2009operatorholderzygmundfunctions}
A.~B. Aleksandrov and V.~V. Peller.
\newblock Operator h\"older--zygmund functions, 2009.
\newblock URL \url{https://arxiv.org/abs/0907.3049}.

\bibitem[Allen-Zhu(2018)]{JMLR:v18:16-410}
Zeyuan Allen-Zhu.
\newblock Katyusha: The first direct acceleration of stochastic gradient
  methods.
\newblock \emph{Journal of Machine Learning Research}, 18\penalty0
  (221):\penalty0 1--51, 2018.
\newblock URL \url{http://jmlr.org/papers/v18/16-410.html}.

\bibitem[Blanchard and M{\"u}cke(2017)]{Muecke2017op.rates}
Gilles Blanchard and Nicole M{\"u}cke.
\newblock Optimal rates for regularization of statistical inverse learning
  problems.
\newblock \emph{Foundations of Computational Mathematics}, 18:\penalty0
  971--1013, 2017.

\bibitem[Caponnetto and De~Vito(2007)]{Caponetto}
A.~Caponnetto and Ernesto De~Vito.
\newblock Optimal rates for the regularized least-squares algorithm.
\newblock \emph{Foundations of Computational Mathematics}, 7:\penalty0
  331--368, 2007.

\bibitem[Carmeli et~al.(2008)Carmeli, Vito, and A.~Toigo]{vvk2}
C.~Carmeli, E.~De Vito, and V.~Umanità A.~Toigo.
\newblock Vector valued reproducing kernel hilbert spaces and universality,
  2008.

\bibitem[Carmeli et~al.(2005)Carmeli, Vito, and Toigo]{vvk1}
Claudio Carmeli, Ernesto~De Vito, and Alessandro Toigo.
\newblock Reproducing kernel hilbert spaces and mercer theorem, 2005.

\bibitem[Carratino et~al.(2019)Carratino, Rudi, and Rosasco]{SGDfeatures}
Luigi Carratino, Alessandro Rudi, and Lorenzo Rosasco.
\newblock Learning with sgd and random features, 2019.

\bibitem[Cortes et~al.(2010)Cortes, Mohri, and Talwalkar]{pmlr-v9-cortes10a}
Corinna Cortes, Mehryar Mohri, and Ameet Talwalkar.
\newblock On the impact of kernel approximation on learning accuracy.
\newblock In Yee~Whye Teh and Mike Titterington, editors, \emph{Proceedings of
  the Thirteenth International Conference on Artificial Intelligence and
  Statistics}, volume~9 of \emph{Proceedings of Machine Learning Research},
  pages 113--120, Chia Laguna Resort, Sardinia, Italy, 13--15 May 2010. PMLR.

\bibitem[Domingos(2020)]{domingos2020model}
Pedro Domingos.
\newblock Every model learned by gradient descent is approximately a kernel
  machine, 2020.

\bibitem[Drineas et~al.(2005)Drineas, Mahoney, and
  Cristianini]{drineas2005nystrom}
Petros Drineas, Michael~W Mahoney, and Nello Cristianini.
\newblock On the nystr{\"o}m method for approximating a gram matrix for
  improved kernel-based learning.
\newblock \emph{journal of machine learning research}, 6\penalty0 (12), 2005.

\bibitem[Engl et~al.(1996)Engl, Hanke, and Neubauer]{engl1996regularization}
Heinz~Werner Engl, Martin Hanke, and Andreas Neubauer.
\newblock \emph{Regularization of inverse problems}, volume 375.
\newblock Springer Science \& Business Media, 1996.

\bibitem[Gadat et~al.(2018)Gadat, Panloup, and Saadane]{10.1214/18-EJS1395}
S{\'e}bastien Gadat, Fabien Panloup, and Sofiane Saadane.
\newblock {Stochastic heavy ball}.
\newblock \emph{Electronic Journal of Statistics}, 12\penalty0 (1):\penalty0
  461 -- 529, 2018.
\newblock \doi{10.1214/18-EJS1395}.
\newblock URL \url{https://doi.org/10.1214/18-EJS1395}.

\bibitem[Gerfo et~al.(2008)Gerfo, Rosasco, Odone, Vito, and
  Verri]{10.1162/neco.2008.05-07-517}
L.~Lo Gerfo, L.~Rosasco, F.~Odone, E.~De Vito, and A.~Verri.
\newblock Spectral algorithms for supervised learning.
\newblock \emph{Neural Computation}, 20\penalty0 (7):\penalty0 1873--1897,
  2008.
\newblock ISSN 0899-7667.
\newblock \doi{10.1162/neco.2008.05-07-517}.
\newblock URL \url{https://doi.org/10.1162/neco.2008.05-07-517}.

\bibitem[Ghadimi et~al.(2015)Ghadimi, Feyzmahdavian, and Johansson]{7330562}
Euhanna Ghadimi, Hamid~Reza Feyzmahdavian, and Mikael Johansson.
\newblock Global convergence of the heavy-ball method for convex optimization.
\newblock In \emph{2015 European Control Conference (ECC)}, pages 310--315,
  2015.
\newblock \doi{10.1109/ECC.2015.7330562}.

\bibitem[Huang et~al.(2024)Huang, Nelsen, and Trautner]{huang2024operator}
Daniel~Zhengyu Huang, Nicholas~H. Nelsen, and Margaret Trautner.
\newblock An operator learning perspective on parameter-to-observable maps,
  2024.

\bibitem[Jacot et~al.(2018)Jacot, Hongler, and Gabriel]{jacot2018neural}
Arthur Jacot, Cl{\'e}ment Hongler, and Franck Gabriel.
\newblock Neural tangent kernel: Convergence and generalization in neural
  networks.
\newblock In \emph{NeurIPS}, 2018.

\bibitem[Kovachki et~al.(2023)Kovachki, Li, Liu, Azizzadenesheli, Bhattacharya,
  Stuart, and Anandkumar]{Kovachki2023NeuralOL}
Nikola~B. Kovachki, Zong-Yi Li, Burigede Liu, Kamyar Azizzadenesheli, Kaushik
  Bhattacharya, Andrew~M. Stuart, and Anima Anandkumar.
\newblock Neural operator: Learning maps between function spaces with
  applications to pdes.
\newblock \emph{J. Mach. Learn. Res.}, 24:\penalty0 89:1--89:97, 2023.
\newblock URL \url{https://api.semanticscholar.org/CorpusID:259149906}.

\bibitem[Kovachki et~al.(2024)Kovachki, Lanthaler, and
  Stuart]{kovachki2024operator}
Nikola~B. Kovachki, Samuel Lanthaler, and Andrew~M. Stuart.
\newblock Operator learning: Algorithms and analysis, 2024.

\bibitem[Lanthaler and Nelsen(2023)]{lanthaler2023error}
Samuel Lanthaler and Nicholas~H. Nelsen.
\newblock Error bounds for learning with vector-valued random features, 2023.

\bibitem[Li et~al.(2015)Li, Bi, Kwok, and Lu]{6918503}
Mu~Li, Wei Bi, James~T. Kwok, and Bao-Liang Lu.
\newblock Large-scale nyström kernel matrix approximation using randomized
  svd.
\newblock \emph{IEEE Transactions on Neural Networks and Learning Systems},
  26\penalty0 (1):\penalty0 152--164, 2015.

\bibitem[Li et~al.(2021{\natexlab{a}})Li, Nica, and Roy]{Li21}
Mufan~Bill Li, Mihai Nica, and Daniel~M. Roy.
\newblock The future is log-gaussian: Resnets and their
  infinite-depth-and-width limit at initialization, 2021{\natexlab{a}}.

\bibitem[Li et~al.(2021{\natexlab{b}})Li, Ton, Oglic, and
  Sejdinovic]{li2021unified}
Zhu Li, Jean-Francois Ton, Dino Oglic, and Dino Sejdinovic.
\newblock Towards a unified analysis of random fourier features,
  2021{\natexlab{b}}.

\bibitem[Lin and Cevher(2018)]{spectral.rates}
Junhong Lin and Volkan Cevher.
\newblock Optimal convergence for distributed learning with stochastic gradient
  methods and spectral algorithms, 2018.
\newblock URL \url{https://arxiv.org/abs/1801.07226}.

\bibitem[Lin et~al.(2020)Lin, Rudi, Rosasco, and Cevher]{Lin_2020}
Junhong Lin, Alessandro Rudi, Lorenzo Rosasco, and Volkan Cevher.
\newblock Optimal rates for spectral algorithms with least-squares regression
  over hilbert spaces.
\newblock \emph{Applied and Computational Harmonic Analysis}, 48\penalty0
  (3):\penalty0 868--890, 2020.

\bibitem[Munteanu et~al.(2022)Munteanu, Omlor, Song, and Woodruff]{Munteanu22}
Alexander Munteanu, Simon Omlor, Zhao Song, and David~P. Woodruff.
\newblock Bounding the width of neural networks via coupled initialization -- a
  worst case analysis, 2022.

\bibitem[Nguyen and Mücke(2023)]{nguyen2023neurons}
Mike Nguyen and Nicole Mücke.
\newblock How many neurons do we need? a refined analysis for shallow networks
  trained with gradient descent, 2023.

\bibitem[Nguyen and Mücke(2024)]{nguyen2024optimalconvergenceratesneural}
Mike Nguyen and Nicole Mücke.
\newblock Optimal convergence rates for neural operators, 2024.
\newblock URL \url{https://arxiv.org/abs/2412.17518}.

\bibitem[Nitanda and Suzuki(2020)]{nitanda2020optimal}
Atsushi Nitanda and Taiji Suzuki.
\newblock Optimal rates for averaged stochastic gradient descent under neural
  tangent kernel regime.
\newblock In \emph{International Conference on Learning Representations}.
  arXiv, 2020.

\bibitem[Oymak and Soltanolkotabi(2019)]{Oymak}
Samet Oymak and Mahdi Soltanolkotabi.
\newblock Towards moderate overparameterization: global convergence guarantees
  for training shallow neural networks, 2019.

\bibitem[Pagliana and Rosasco(2019)]{pagliana2019implicit}
Nicol{\`o} Pagliana and Lorenzo Rosasco.
\newblock Implicit regularization of accelerated methods in hilbert spaces.
\newblock \emph{Advances in Neural Information Processing Systems},
  32:\penalty0 14481--14491, 2019.

\bibitem[Pillaud-Vivien et~al.(2018)Pillaud-Vivien, Rudi, and
  Bach]{pillaudvivien2018statistical}
Loucas Pillaud-Vivien, Alessandro Rudi, and Francis Bach.
\newblock Statistical optimality of stochastic gradient descent on hard
  learning problems through multiple passes, 2018.

\bibitem[Rahimi and Recht(2007)]{NIPS2007_013a006f}
Ali Rahimi and Benjamin Recht.
\newblock Random features for large-scale kernel machines.
\newblock In \emph{Advances in Neural Information Processing Systems}. Curran
  Associates, Inc., 2007.

\bibitem[Rahimi and Recht(2008)]{NIPS2008_0efe3284}
Ali Rahimi and Benjamin Recht.
\newblock Weighted sums of random kitchen sinks: Replacing minimization with
  randomization in learning.
\newblock In D.~Koller, D.~Schuurmans, Y.~Bengio, and L.~Bottou, editors,
  \emph{Advances in Neural Information Processing Systems}, volume~21. Curran
  Associates, Inc., 2008.
\newblock URL
  \url{https://proceedings.neurips.cc/paper_files/paper/2008/file/0efe32849d230d7f53049ddc4a4b0c60-Paper.pdf}.

\bibitem[Roux et~al.(2012)Roux, Schmidt, and Bach]{NIPS2012_905056c1}
Nicolas Roux, Mark Schmidt, and Francis Bach.
\newblock A stochastic gradient method with an exponential convergence \_rate
  for finite training sets.
\newblock In F.~Pereira, C.J. Burges, L.~Bottou, and K.Q. Weinberger, editors,
  \emph{Advances in Neural Information Processing Systems}, volume~25. Curran
  Associates, Inc., 2012.
\newblock URL
  \url{https://proceedings.neurips.cc/paper_files/paper/2012/file/905056c1ac1dad141560467e0a99e1cf-Paper.pdf}.

\bibitem[Rudi and Rosasco(2016)]{features}
Alessandro Rudi and Lorenzo Rosasco.
\newblock Generalization properties of learning with random features, 2016.

\bibitem[Rudi et~al.(2016)Rudi, Camoriano, and Rosasco]{rudi2016more}
Alessandro Rudi, Raffaello Camoriano, and Lorenzo Rosasco.
\newblock Less is more: Nystroem computational regularization, 2016.

\bibitem[Schoelkopf and Smola(2002)]{kernellearning}
B.~Schoelkopf and A.~J. Smola.
\newblock \emph{Learning with Kernels, Support Vector Machines, Regularization,
  Optimization, and Beyond (Adaptive Computation and Machine Learning)}.
\newblock MIT Press, 2002.

\bibitem[Shao(2003)]{Shao_2003_book}
Jun Shao.
\newblock \emph{Mathematical Statistics}.
\newblock Springer-Verlag New York Inc, 2nd edition, 2003.

\bibitem[Steinwart and Christmann(2008{\natexlab{a}})]{Ingo}
Ingo Steinwart and Andreas Christmann.
\newblock \emph{Support vector machines}.
\newblock Springer Science \& Business Media, 2008{\natexlab{a}}.

\bibitem[Steinwart and Christmann(2008{\natexlab{b}})]{steinwart2008support}
Ingo Steinwart and Andreas Christmann.
\newblock \emph{Support vector machines}.
\newblock Springer Science \& Business Media, 2008{\natexlab{b}}.

\bibitem[Sterge and Sriperumbudur(2022)]{JMLR:v23:21-0766}
Nicholas Sterge and Bharath~K. Sriperumbudur.
\newblock Statistical optimality and computational efficiency of nystrom kernel
  pca.
\newblock \emph{Journal of Machine Learning Research}, 23\penalty0
  (337):\penalty0 1--32, 2022.
\newblock URL \url{http://jmlr.org/papers/v23/21-0766.html}.

\bibitem[Sterge et~al.(2020)Sterge, Sriperumbudur, Rosasco, and
  Rudi]{pmlr-v108-sterge20a}
Nicholas Sterge, Bharath Sriperumbudur, Lorenzo Rosasco, and Alessandro Rudi.
\newblock Gain with no pain: Efficiency of kernel-pca by nyström sampling.
\newblock In Silvia Chiappa and Roberto Calandra, editors, \emph{Proceedings of
  the Twenty Third International Conference on Artificial Intelligence and
  Statistics}, volume 108 of \emph{Proceedings of Machine Learning Research},
  pages 3642--3652. PMLR, 26--28 Aug 2020.

\bibitem[Tropp(2011)]{Tropp_2011}
Joel~A. Tropp.
\newblock User-friendly tail bounds for sums of random matrices.
\newblock \emph{Foundations of Computational Mathematics}, 12\penalty0
  (4):\penalty0 389--434, 2011.

\bibitem[Williams and Seeger(2000)]{NIPS2000_19de10ad}
Christopher Williams and Matthias Seeger.
\newblock Using the nystr\"{o}m method to speed up kernel machines.
\newblock In T.~Leen, T.~Dietterich, and V.~Tresp, editors, \emph{Advances in
  Neural Information Processing Systems}, volume~13. MIT Press, 2000.
\newblock URL
  \url{https://proceedings.neurips.cc/paper_files/paper/2000/file/19de10adbaa1b2ee13f77f679fa1483a-Paper.pdf}.

\bibitem[Xiao and Zhang(2014)]{doi:10.1137/140961791}
Lin Xiao and Tong Zhang.
\newblock A proximal stochastic gradient method with progressive variance
  reduction.
\newblock \emph{SIAM Journal on Optimization}, 24\penalty0 (4):\penalty0
  2057--2075, 2014.
\newblock \doi{10.1137/140961791}.

\bibitem[Zhang et~al.(2024)Zhang, Li, and
  Lin]{zhang2024optimalitymisspecifiedspectralalgorithms}
Haobo Zhang, Yicheng Li, and Qian Lin.
\newblock On the optimality of misspecified spectral algorithms, 2024.
\newblock URL \url{https://arxiv.org/abs/2303.14942}.

\bibitem[Zhen et~al.(2020)Zhen, Sun, Du, Xu, Yin, Shao, and
  Snoek]{pmlrv119zhen20a}
Xiantong Zhen, Haoliang Sun, Yingjun Du, Jun Xu, Yilong Yin, Ling Shao, and
  Cees Snoek.
\newblock Learning to learn kernels with variational random features, 2020.

\end{thebibliography}
